\documentclass[nohyperref]{article}
\usepackage{microtype}
\usepackage{graphicx}
\usepackage{booktabs} 
\usepackage{multirow}
\usepackage{bbm}
\usepackage{multicol}
\usepackage{caption}
\usepackage{subcaption}
\usepackage{makecell}
\usepackage{xspace}
\usepackage{hyperref}

 \usepackage[accepted]{icml2023}

\usepackage{amsmath}
\usepackage{amssymb}
\usepackage{mathtools}
\usepackage{amsthm}

\usepackage[capitalize,noabbrev]{cleveref}

\theoremstyle{plain}
\newtheorem{theorem}{Theorem}[section]

\newtheorem{lemma}[theorem]{Lemma}

\theoremstyle{definition}
\newtheorem{definition}[theorem]{Definition}
\newtheorem{assumption}[theorem]{Assumption}
\theoremstyle{remark}
\newtheorem{remark}[theorem]{Remark}

\usepackage[textsize=tiny]{todonotes}

\usepackage{amsmath,amsfonts,bm}

\def\eqref#1{equation~\ref{#1}}

\def\1{\bm{1}}

\def\vone{{\bm{1}}}

\DeclareMathAlphabet{\mathsfit}{\encodingdefault}{\sfdefault}{m}{sl}
\SetMathAlphabet{\mathsfit}{bold}{\encodingdefault}{\sfdefault}{bx}{n}

\def\sR{{\mathbb{R}}}

\newcommand{\htheta}{\hat{\theta}}
\newcommand{\ttheta}{\tilde{\theta}}
\newcommand{\eval}[1]{\Bigr|_{#1}}
\newcommand{\ours}{ROCERF\xspace}

\icmltitlerunning{Towards Bridging the Gaps between the Right to Explanation and the Right to be Forgotten}

\begin{document}
\onecolumn

\icmltitle{Towards Bridging the Gaps between the Right to Explanation\\ and the Right to be Forgotten}

\icmlsetsymbol{equal}{*}

\begin{icmlauthorlist}
\icmlauthor{Satyapriya Krishna}{equal,harvard}
\icmlauthor{Jiaqi Ma}{equal,harvard}
\icmlauthor{Himabindu Lakkaraju}{harvard}
\end{icmlauthorlist}

\icmlaffiliation{harvard}{Harvard University}

\icmlcorrespondingauthor{Himabindu Lakkaraju}{hlakkaraju@hbs.edu}

\icmlkeywords{Counterfactual Explanations, Algorithmic Recourse, Right to Explain, Right to Be Forgotten}

\vskip 0.3in



\printAffiliationsAndNotice{\icmlEqualContribution} 

\begin{abstract}
\emph{The Right to Explanation} and \emph{the Right to be Forgotten} are two important principles outlined to regulate algorithmic decision making and data usage in real-world applications. While the right to explanation allows individuals to request an actionable explanation for an algorithmic decision, the right to be forgotten grants them the right to ask for their data to be deleted from all the databases and models of an organization. Intuitively, enforcing the right to be forgotten may trigger model updates which in turn invalidate previously provided explanations, thus violating the right to explanation. In this work, we investigate the technical implications arising due to the interference between the two aforementioned regulatory principles, and propose \emph{the first algorithmic framework} to resolve the tension between them. To this end, we formulate a novel optimization problem to generate explanations that are robust to model updates due to the removal of training data instances by data deletion requests. We then derive an efficient approximation algorithm to handle the combinatorial complexity of this optimization problem. We theoretically demonstrate that our method generates explanations that are provably robust to worst-case data deletion requests with bounded costs in case of linear models and certain classes of non-linear models. Extensive experimentation with real-world datasets demonstrates the efficacy of the proposed framework. 
\end{abstract}

\vspace{-0.3in}
\section{Introduction}
Over the past decade, machine learning models have been increasingly deployed in various high-stakes decision making scenarios including hiring and loan approvals. Consequently, a number of regulatory policies and principles~\citep{GDPR,CCPA} were introduced to ensure that algorithmic decisions and data usage practices in real-world applications do not cause any undue harm to individuals. \emph{The Right to Explanation} and \emph{the Right to be Forgotten} are two such notable regulatory principles which were first introduced by the European Union’s General Data Protection Regulation (GDPR)~\citep{GDPR}. While the right to explanation ensures that individuals who are negatively impacted by adverse algorithmic outcomes are provided with an
actionable explanation, the right to be forgotten ensures that individuals have the right to ask for their data to be removed from all the databases and models of an organization.

To operationalize the right to explanation in practice, several strategies have been considered in recent literature. A particular class of explanations commonly referred to as \emph{counterfactual explanations} or \emph{algorithmic recourse} are often considered very promising in this regard. For instance, when an individual is denied a loan by a predictive model employed by a bank, a counterfactual explanation (or an algorithmic recourse) provides them with inputs about what aspects (features) of their profile should be changed and by how much in order to obtain a positive outcome. Several approaches in recent literature tackled the problem of generating such counterfactual explanations~\citep{wachter2017counterfactual,Ustun2019ActionableRI,pawelczyk2020learning,karimi2019model}. 

Prior research has also explored various strategies to operationalize the right to be forgotten~\citep{Cao2015-bi,Ginart2019-xh,Garg2020-nw}. Since the right to be forgotten requires organizations to delete pertinent user data from all their databases and models, it often involves retraining or updating their models. To this end, several methods were proposed to efficiently update machine learning models in the face of (training) data deletion requests without having to retrain them from scratch~\citep{Guo2020-ml,Bourtoule2021-tz,Izzo2021-sm,Neel2021-iu}. 

Despite the significance of the two aforementioned regulatory principles, there is very little research that explores potential interference between them. Intuitively, enforcing the right to be forgotten may trigger model updates which in turn invalidate previously provided actionable explanations that end users may act upon, thus violating the right to explanation. For instance, consider a scenario where a user was asked to increase their salary by 5K to get a loan and they start working towards it, but the underlying model gets updated in the meanwhile to accommodate (training) data deletion requests. Consequently, the user may no longer receive the desired outcome even if their salary increases by 5K as the previously prescribed recourse may no longer hold with respect to the new model. \citet{Pawelczyk2022-qv} highlighted this challenge and argued that the right to explanation and the right to be forgotten are in conflict with each other, and that existing methods are not capable of dealing with this tension. 

In this work, we make one of the first attempts to resolve the aforementioned tension and bridge the operational gaps between the right to explanation and the right to be forgotten. More specifically, we propose the \emph{first algorithmic framework}, RObust Counterfactual Explanations under the Right to be Forgotten (\ours), to address this problem. To this end, we formulate a novel optimization problem to generate counterfactual explanations that remain valid in the face of model updates (changes) arising due to (training) data deletion requests. This optimization problem turns out to be combinatorially complex as it considers $n$ training instances and $k$ data deletion requests resulting in $n\choose k$ possible ways of the model being updated. To mitigate this computational challenge, we propose a novel algorithm which can efficiently approximate model updates relative to the original model, and select those with most significant deviations, thus eliminating the need for retraining $n\choose k$ models. With this approximation, we are able to develop a practically efficient algorithm to learn effective counterfactual explanations that remain valid on model updates triggered by data deletion requests.

We theoretically and empirically analyze the validity and costs of the counterfactual explanations generated by our framework \ours. In case of linear models and non-linear models with certain regularity assumptions, we theoretically demonstrate that our method generates counterfactual explanations that are provably valid in the face of worst-case data deletion requests, while incurring additional costs upper bounded by $O(\frac{k}{n})$.
Empirically, we evaluate the proposed \ours and state-of-the-art counterfactual explanation methods using logistic regression and neural network models on three real-world datasets. The proposed method outperforms baseline methods in most experimental settings.
In comparison, baseline methods either fail dramatically in terms of validity, or achieve high validity with significantly higher cost. 
Our results establish that our framework \ours enables us to simultaneously enforce both the right to explanation as well as the right to be forgotten, thus bridging a critical operational gap between the two regulatory principles.
\vspace{-0.1in}
\section{Related Work}
Over the past few years, there has been a lot of exciting research on counterfactual explanations or algorithmic recourse~\citep{tolomei2017interpretable,laugel2017inverse,Wachter2017-yk,Ustun2019ActionableRI,van2019interpretable,mahajan2019preserving,mothilal2020fat,karimi2019model,rawal2020interpretable,dandl2020multi}. Several of the proposed approaches can be roughly categorized
along the following dimensions \citep{verma2020counterfactual}: 
\emph{type of the underlying predictive model} (e.g., tree based vs.\ differentiable classifier), whether they encourage \emph{sparsity} in counterfactuals (i.e., only a small number of features should be changed), whether counterfactuals should lie on the \emph{data manifold} and whether the underlying \emph{causal relationships} should be accounted for when generating counterfactuals. Most of these approches assume that the underlying predictive model remains unchanged before and after the end users implement the prescribed recourses.

More recently, few studies have investigated the impact of changes in the underlying predictive models on the the validity of recourses~\citep{rawal2021modelshifts,Upadhyay2021-ew}. To improve the robustness of the recourses in the face of such model changes, prior work has proposed adversarial training methods that generate counterfactual explanations robust to small (and often Gaussian) perturbations of the underlying model parameters~\citep{Upadhyay2021-ew}. While such methods could potentially be considered to mitigate the challenges brought about by the right to be forgotten, it is unclear how the removal of training data points will affect the model parameters. There is no guarantee that counterfactual explanations robust to Gaussian perturbations of model parameters will be valid under model updates (changes) occurring due to data deletion requests. 

On the other hand, the right to be forgotten has also inspired considerable research in machine learning literature~\citep{Cao2015-bi,Ginart2019-xh,Garg2020-nw,Guo2020-ml,Bourtoule2021-tz,Izzo2021-sm,Neel2021-iu}. Majority of work along these lines focuses on developing methods to efficiently update models in the face of training data deletion requests, without having to retrain models from scratch. Such approaches are referred to as \emph{Machine Unlearning} methods. 

To the best of our knowledge, the only prior work at the intersection of the right to explanation and the right to be forgotten is by \citet{Pawelczyk2022-qv}. \citet{Pawelczyk2022-qv} analyzed the impact of (training) data deletion requests on the validity of counterfactual explanations generated by existing methods, and concluded that the explanations generated by state-of-the-art methods become invalid in the face of model updates due to data deletion requests. While the above work highlighted the tension between the right to explanation and the right to be forgotten, they do not provide a solution to this critical problem. In contrast, our work proposes the first algorithmic framework to address this tension.  

\vspace{-0.1in}
\section{Our Framework ROCERF}
\label{sec:method}

In this section, we introduce our framework, RObust Counterfactual Explanations under the Right to be Forgotten (ROCERF). Specifically, we first formally define the problem of finding robust counterfactual explanations in the presence of training data removal required by the right to be forgotten. Then, we present an efficient approximation algorithm to solve this problem. We also discuss practical considerations including computation costs and further approximations in implementation. 

\subsection{Problem Definition}
\label{sec:method-problem-def}

Suppose we have a training dataset $D=\{(x_i, y_i)\}_{i=1}^n$ with $n$ data points, where $x_i\in \mathcal{X}, y_i \in \{-1, 1\}$ are respectively features and labels. Given a family of classifiers $f_{\theta}: \mathcal{X}\rightarrow \sR$ parameterized by $\theta \in \Theta$, a classifier $f_{\theta}$ predicts $1$ on a data point $x$ if $f_{\theta}(x) \ge 0$ and predicts $-1$ otherwise. Assume $B:=\sup_{x\in \mathcal{X}}\|x\|_2$ is $O(1)$ in terms of $n$.

To characterize the data removal, we introduce a data weight vector $w\in \{0, 1\}^n$. For each data point $i$, let $w_i = 0$ if this data point is removed, and let $w_i=1$ otherwise. Specially, when $w=\vone$, where $\vone$ is an all-one vector, there is no data point being removed.

Denote the loss of $f_{\theta}$ on each data point $i$ as $l_i(\theta)$ and assume $l_i(\theta)$ has continuous second derivatives. We denote the classifier trained on the dataset $D$ as $f_{\htheta_{\vone}}$, where $\htheta_{\vone} = \arg\min_{\theta\in \Theta} \frac{1}{n}\sum_{i=1}^n l_i(\theta)$. A classifier trained on the dataset with some removals indicated by $w$ can then be denoted as $f_{\htheta_w}$, where $\htheta_w = \arg\min_{\theta\in \Theta} \frac{1}{\|w\|_1} \sum_{i=1}^n w_i l_i(\theta)$. To simplify notations, for $i = 1, 2, \ldots n$, define $g_i(\theta) := \frac{\partial l_i(\theta)}{\partial \theta}$ and $h_i(\theta) := \frac{\partial g_i(\theta)}{\partial \theta^T}$. Then $H := \frac{1}{n} \sum_{i=1}^n h_i(\htheta_{\vone})$ is the Hessian matrix of the loss function on the whole dataset $D$. 

In the literature~\citep{Wachter2017-yk,Verma2020-bp}, the problem of finding counterfactual explanations (CFEs) for the model $f_{\theta_{\vone}}$ trained on the original full dataset is often defined as an optimization problem like the following Definition~\ref{def:cfe}.

\begin{definition}[Counterfactual Explanation (CFE)]
\label{def:cfe}
    For any data point $x_0\in \mathcal{X}$, the \emph{CFE ($\tilde{x}_0\in \mathcal{X}$)} of $x_0$, is defined as the solution of the following optimization problem 
    \vspace{-0.08in}
    \begin{align}
        \min_{x\in \mathcal{X}}\quad& \|x - x_0\|_2 \label{eq:cfe} \\
        \textrm{subject to}\quad& f_{\htheta_{\vone}}(x) \ge 0.\nonumber
    \end{align}
\end{definition}
Intuitively, we hope to find a \emph{valid} CFE (classified as $1$) with the minimum \emph{cost}, as measured by the L2 distance to $x_0$. In practice, the L2 distance could be replaced by other distance functions, such as L1 distance or any other metrics suitable for the application. In this paper, however, we stick to the L2 distance following the convention of recent literature~\citep{pawelczyk2021carla}.

In this paper, we aim to obtain CFEs that is robustly valid against potential data point removal required by right to be forgotten. To formalize this problem, we define the following \emph{$k$-Removal-Robust CFE ($k$RR-CFE)} that is supposed to be robust with respect to any removal of $k$ data points. 

\begin{definition}[$k$-Removal-Robust CFE ($k$RR-CFE)]
\label{def:krr-cfe}
    Given an integer $k>0$, denote the set of all possible weight vectors with $k$ data removals as $\mathcal{W}^{(k)} = \{w \in \{0, 1\}^n: \|w\|_1 = n-k\}$. For any data point $x_0\in \mathcal{X}$, the \emph{$k$-RR CFE ($\tilde{x}_0^{(k)}\in \mathcal{X}$)} of $x_0$, is defined as the solution of the following optimization problem 
    \begin{align}
        \min_{x\in \mathcal{X}}\quad& \|x - x_0\|_2 \label{eq:krr-cfe} \\
        \textrm{subject to}\quad& f_{\htheta_w}(x) \ge 0, \forall w\in \mathcal{W}^{(k)}.\nonumber
    \end{align}
\end{definition}

While the $k$-RR CFE defined in Definition~\ref{def:krr-cfe} is robust to any removal of $k$ data points by construction, a naive implementation to obtain the $k$-RR CFE requires one to retrain $|\mathcal{W}^{(k)}| = {n \choose k}$ classifiers that appear in the constraint of the optimization problem (\ref{eq:krr-cfe}), which is computationally impractical.

\subsection{Approximating $k$-RR CFE}
\label{sec:approx-k-RR}

To address the computational challenge, we propose an efficient algorithm to approximate $k$-RR CFE. The proposed method first efficiently approximates the classifier $f_{\htheta_w}$, for any $ w\in \mathcal{W}^{(k)}$, without the need of retraining from scratch. Then we show that, we can reduce the constraint set with $n\choose k$ classifiers to an equivalent constraint that only requires a linear computation complexity with respect to $n$. 

\paragraph{Approximating the Classifier.}
A key observation that makes it possible to efficiently approximate $f_{\htheta_w}$ is that these classifiers can be viewed as \emph{leave-$k$-out (LKO)} estimators, which can be efficiently approximated by leveraging recent advances in LKO analysis~\citep{Giordano2019-mj,Broderick2020-lg}. 

For each of the classifier $f_{\htheta_w}(x)$ in the constraint set of the problem~\ref{eq:krr-cfe}, note that $f_{\htheta_w}(x)$ is also a function of $w$\footnote{Strictly speaking, we need to assume uniqueness of $\htheta_w$ for any given $w$. Although in practice we can often make this assumption approximately hold locally.}. For any fixed $x$, we can take a first-order Taylor approximation of $f_{\htheta_w}(x)$ with respect to $w$ at $w=\vone$, and denote this first-order approximation as $\tilde{f}_{\htheta_w}(x)$, i.e.,
\begin{align}
    \tilde{f}_{\htheta_w}(x) &= f_{\htheta_{\vone}}(x) + \frac{\partial f_{\htheta_w}(x)}{\partial w} \eval{w=\vone} (w - \vone) \label{eq:taylor} \\
    &= f_{\htheta_{\vone}}(x) + \frac{\partial f_{\theta}(x)}{\partial \theta} \eval{\theta=\htheta_{\vone}} \frac{\partial \htheta_{w}}{\partial w} \eval{w=\vone} (w - \vone), \label{eq:taylor-chain}
\end{align}
where from Eq.~(\ref{eq:taylor}) to Eq.~(\ref{eq:taylor-chain}), we have applied the chain rule. Using results in \citet{Giordano2019-mj}, we can show that\footnote{See Proposition 3 in Appendix A.3 of \citet{Giordano2019-mj}.} 
\begin{align}
    \frac{\partial \htheta_{w}}{\partial w} \eval{w=\vone} (w - \vone) = \frac{1}{n}\sum_{i:w_i=0} H^{-1} g_i(\htheta_{\vone}). \label{eq:Hg}  
\end{align} 
Therefore, the first-order Taylor approximation can be written as
\begin{align}
    \tilde{f}_{\htheta_w}(x) = f_{\htheta_{\vone}}(x) + \frac{1}{n} \sum_{i:w_i=0} \beta(x)^T  H^{-1} g_i(\htheta_{\vone}), \label{eq:taylor-explicit-form}
\end{align}
where $\beta(x) := \left(\frac{\partial f_{\theta}(x)}{\partial \theta} \eval{\theta=\htheta_{\vone}}\right)^T$ is the gradient of $f_{\theta}(x)$ with respect to the model parameters $\theta$ at $\theta = \htheta_{\vone}$. 

Replacing $f_{\htheta_w}$ with $\tilde{f}_{\htheta_w}$, we approximate the problem~(\ref{eq:krr-cfe}) with a new problem below.
\begin{align}
    \min_{x\in \mathcal{X}}\quad& \|x - x_0\|_2 \label{eq:krr-cfe-approximate} \\
    \textrm{subject to}\quad& \tilde{f}_{\htheta_w}(x) \ge \delta, \forall w\in \mathcal{W}^{(k)},\nonumber
\end{align}
where $\delta > 0$ is a constant accounting for the approximation error of $\tilde{f}_{\htheta_w}$ for $f_{\htheta_w}$, which should be chosen in a way such that for any $w\in \mathcal{W}^{(k)}$ and $x\in \mathcal{X}$, $\tilde{f}_{\htheta_w}(x) \ge \delta$ implies $f_{\htheta_w}(x) \ge 0$. The proper choice of $\delta$ is model dependent and, in practice, can be treated as a hyperparameter selected using a validation set. We also provide some theoretical insights on $\delta$ in Section~\ref{sec:theory}.

\paragraph{Reducing the Constraint Set.} With the first-order Taylor approximate classifier $\tilde{f}_{\htheta_w}$, we can further reduce the constraint set with $n \choose k$ inequalities to a single inequality.

Note that in Eq.~(\ref{eq:taylor-explicit-form}), $f_{\htheta_{\vone}}(x), \beta(x), H$, and $g_i(\htheta_{\vone})$ are all calculated based on the model $f_{\htheta_{\vone}}$ trained on the original full dataset $D$, and are independent of the data weight vector $w$. Define a set $\mathcal{A}(x) := \{\beta(x)^T H^{-1} g_i(\htheta_{\vone})\}$. Then satisfying the constraints in the problem~(\ref{eq:krr-cfe-approximate}) is equivalent to having the following condition.
\begin{align}
    f_{\htheta_{\vone}}(x) + \frac{1}{n}\min_{\mathcal{B} \subseteq \mathcal{A}(x), |\mathcal{B}| = k} \sum_{b\in \mathcal{B}} b \ge \delta, \label{eq:sum-min-nonlinear}
\end{align}
where one shall recall that $w\in \mathcal{W}^{(k)}$ always has $k$ entries as $0$ and the remaining as $1$. 

Defining
\[
f_{\mathcal{A}}^{(k)}(x) := f_{\htheta_{\vone}}(x) + \frac{1}{n}\min_{\mathcal{B} \subseteq \mathcal{A}(x), |\mathcal{B}| = k} \sum_{b\in \mathcal{B}} b,
\]
we have shown that solving the problem~(\ref{eq:krr-cfe-approximate}) is equivalent to solving the problem below.
\begin{align}
    \min_{x\in \mathcal{X}}\quad& \|x - x_0\|_2 \label{eq:krr-cfe-approximate-single-constraint} \\
    \textrm{subject to}\quad& f_{\mathcal{A}}^{(k)}(x) \ge \delta.\nonumber
\end{align}

\paragraph{Optimization.}
To solve the constrained optimization problem~(\ref{eq:krr-cfe-approximate-single-constraint}), we use the the penalty method~\citep{Freund2004-in}. Define the penalty function as $\phi(z) := \max(z, 0)^2$. We solve a series of unconstrained relaxation of the original problem~(\ref{eq:krr-cfe-approximate-single-constraint}):
\begin{align}
    \min_{x\in \mathcal{X}}\quad& J_t(x) = \lambda_t  \phi(\delta - f_{\mathcal{A}}^{(k)}(x)) + \|x - x_0\|_2, \label{eq:krr-cfe-approximate-linear-relax}
\end{align}
for $t=1,2,\ldots T$ as the iteration index. And $\lambda_t\ge 0$ is the penalty coefficient controlling the relative strength between the penalty and the original objective for each iteration $t$. Denote the solution of the $t$-th iteration as $x^*_t$. We start with a small $\lambda_1$ and double it until the first $t_0$ where $f_{\mathcal{A}}^{(k)}(x^*_{t_0}) \ge \delta$ while $f_{\mathcal{A}}^{(k)}(x^*_{t_0-1}) < \delta$. Then we have a binary search on the penalty coefficient between $\lambda_{t_0 - 1}$ and $\lambda_{t_0}$ to obtain a feasible solution with as small cost as possible. Please see Algorithm~\ref{alg:opt} in Appendix~\ref{app:algorithm} for more details.

\subsection{Practical Considerations}
\label{sec:practical}

\paragraph{Computation Costs.}
Finally, we make a few remarks on the computation costs of the proposed method. 
Given the dataset $D$ and the original model $f_{\htheta_{\vone}}$ trained on $D$, we need to first calculate the gradients $g_i(\htheta_{\vone}), i=1,\ldots, n$ and the Hessian inverse $H^{-1}$. Calculating the exact Hessian inverse may be expensive for models with high-dimensional parameters, such as neural networks. However, we can leverage computational tricks calculating influence functions~\citep{Koh2017-vo} to efficiently approximate $H^{-1}$. In addition, $H^{-1}$ and $g_i(\htheta_{\vone})$'s only need to be calculated once for the whole process and are shared for all the test samples.

The major computation cost comes from repeated evaluations of $f_{\mathcal{A}}^{(k)}(x)$ at different $x$ during the optimization procedure. For each $x$, we can use automatic differentiation tools such as PyTorch~\citep{Paszke2017-ce} to evaluate $f_{\htheta_{\vone}}(x)$ and $\beta(x)$ by one forward pass and one backward pass. Suppose we have pre-computed and stored the values of $H^{-1}g_i(\htheta_{\vone}), i=1,\ldots, n$, we can obtain the set $\mathcal{A}(x)$ by $n$ vector multiplications. Finally, evaluating $f_{\mathcal{A}}^{(k)}(x)$ requires partially sorting $\mathcal{A}(x)$ and obtaining the bottom-$k$ values, which has a complexity of $O(n\log k)$. Overall, evaluating the constraint of problem~(\ref{eq:krr-cfe-approximate-single-constraint}) has a complexity that is linear in $n$, which is much smaller than evaluating the original constraint set with $n \choose k$ models. 

\paragraph{Hyperparameters.} The proposed method has two hyperparameters, $k$ and $\delta$. The hyperparameter $k$ should be set from a rough estimate of the number of data removals, which relies on domain knowledge of the application. Empirically, however, we find the method is not very sensitive to the value of $k$ so there is a good tolerance on the choice of $k$. The hyperparameter $\delta$ measures how good is the Taylor approximation of the function. In practice, we can choose $\delta$ on a validation set and simulating a few models trained after random removals. But in our experiments, we find fixing it as 0 also works well empirically.

\paragraph{A Special Case: Linear Models.}
When $f_{\theta}(x) = \theta^T x$ is a linear model, the first-order Taylor approximation in Eq.~(\ref{eq:taylor-explicit-form}) simplifies to the following form:
\begin{align}
    \tilde{f}_{\htheta_{w}}(x) = \htheta_{\vone}^T x + \frac{1}{n}\sum_{i:w_i = 0}x^T H^{-1}g_i(\htheta_{\vone}), \label{eq:taylor-explicit-form-linear}
\end{align}
since $\beta(x) = x$ for linear models. In this special case, we can avoid going through the backward pass when evaluating $\beta(x)$, which makes the optimization much more efficient.

\paragraph{Local Linear Approximation of Nonlinear Models.} Owing to computational efficiency considerations, it is a common practice in recourse literature to first obtain a local linear approximation of the underlying model at each test sample, and then leverage this to compute counterfactual explanations~\citep{Upadhyay2021-ew,Ustun2019ActionableRI,rawal2020interpretable}.
Along similar lines, we propose to apply ROCERF on local linear approximations of nonlinear models to further improve the computational efficiency in practice. Specifically, we use LIME~\citep{ribeiro2016should} to obtain local linear approximations of the underlying models. 

\section{Theoretical Analysis of Validity and Cost}
\label{sec:theory}
In this section, we provide theoretical guarantees on validity and cost of CFEs obtained by the proposed method, under a small fraction of data removal in the training set. In particular, we characterize the trade-off between validity and cost and provide upper bounds on the cost needed to guarantee that the CFE is robustly valid. We first present an analysis for linear models and then for nonlinear models with regularity assumptions.

\subsection{Analysis on Linear Models}
\label{sec:theory-linear}
Assume the machine learning models are regularized logistic regression, i.e., $l_i(\theta) = \log(1+\exp(-y_i\theta^T x_i)) + \gamma \|\theta\|_2^2$, and the model parameters have bounded norm. In this case, the following Theorem~\ref{thm:validity-cost-linear} provides theoretical guarantees on the validity and cost, and the detailed proof of which can be found in Appendix~\ref{app:proof-validity-cost-linear}.
\begin{theorem}[Validity and Cost on Logistic Regression]
\label{thm:validity-cost-linear}
For any data point $x_0\in \mathcal{X}$, let $\tilde{x}_0$ be the CFE of $x_0$, and let $\tilde{x}_0^{(k)}$ be the solution of the optimization problem~(\ref{eq:krr-cfe-approximate-single-constraint}) when the classifiers are regularized logistic regression. Then we can properly choose $\delta$ such that, if $\tilde{x}_0^{(k)}$ exists,  $\tilde{x}_0^{(k)}$ remains a valid CFE for all possible removal of $k$ data points, i.e., 
\[f_{\htheta_w}(\tilde{x}_0^{(k)}) \ge 0, \forall w\in \mathcal{W}^{(k)}.\]

\vspace{-0.1in}
Furthermore, the cost of implementing $\tilde{x}_0^{(k)}$ is upper bounded as following,
\begin{align}
    \|\tilde{x}_0^{(k)} - x_0\|_2 \le \|\tilde{x}_0 - x_0\|_2 + \frac{kC}{n\|\htheta_{\vone}\|_2}, \label{eq:cost-bound-linear}
\end{align}
where $C$ is a constant independent of $n$.
\end{theorem}

\begin{proof}[Proof Sketch]
The proof of Theorem~\ref{thm:validity-cost-linear} involves two key steps. The first step is to derive a bound on the difference between the actual retrained model $f_{\htheta_w}$ and its Taylor approximation model $\tilde{f}_{\htheta_w}$. This bound gives us an estimate on how large $\delta$ is needed in the optimization problem~(\ref{eq:krr-cfe-approximate-single-constraint}) in order to ensure validity. The second step is to derive a bound on the difference between the Taylor approximation model $\tilde{f}_{\htheta_w}$ and the original model $f_{\htheta_{\vone}}$. For regularized logistic regression, both differences can be well bounded without further assumptions. Note that the difference between $\tilde{x}_0^{(k)}$ and $\tilde{x}_0$ is that the former is constrained by $\tilde{f}_{\htheta_w}(\tilde{x}_0^{(k)}) \ge \delta$ while the latter is constrained by $f_{\htheta_{\vone}}(\tilde{x}_0) \ge 0$. So together with the estimate on $\delta$, the bound on the difference between $\tilde{f}_{\htheta_w}$ and $f_{\htheta_{\vone}}$ allows us to bound the additional cost of $\tilde{x}_0^{(k)}$ in comparison to $\tilde{x}_0$. 
\end{proof}

This result states that the additional cost needed to achieve robust validity has an upper bound of $O(\frac{k}{n})$, and this additional cost vanishes when the training set size $n$ is very large and the number of removals $k$ is relatively small. As a sanity check, in the degenerate case where there is no data removed, i.e., $k=0$, there is also no additional cost.

This result also indicates that for simple models trained on abundant data, it is possible to provide robustly valid recourses to users with little additional costs, thus paving the way for \emph{bridging critical operational gaps between the right to explanation and the right to be forgotten}. The technical insight behind this strong guarantee is that, when the number of data removals is not too large compared to the training set, the retrained model will not change too much (difference between $f_{\htheta_w}$ and $f_{\htheta_{\vone}}$), and the change can be efficiently estimated (through $\tilde{f}_{\htheta_w}$). 

\begin{remark}
Technically, neither the problem~(\ref{eq:krr-cfe}) nor its approximation~(\ref{eq:krr-cfe-approximate-single-constraint}) is guaranteed to be feasible. However, especially for linear models, we find that they are always feasible on the datasets we empirically tested. This is possibly because the difference among $f_{\htheta_w}$ for all $w\in \mathcal{W}^{(k)}$ is not dramatically large.
\end{remark}

\subsection{Analysis on Nonlinear Models}
\label{sec:theory-nonlinear}

Next, we generalize Theorem~\ref{thm:validity-cost-linear} to nonlinear models with the following assumptions.

\begin{assumption}
\label{assump:regularity}
Assume that there exist universal finite constants $C_1, C_2, C_3, C_4, C_5$ independent of $n$ such that 
\begin{enumerate}
    \item $\sup_{\theta\in \Theta}\frac{1}{n}\sum_{i=1}^n \|h_i(\theta)\|_{F} \le C_1$;
    \item $\sup_{\theta\in \Theta} \|g_i(\theta)\|_{2} \le C_2, i=1,\ldots, n$;
    \item $\sup_{x\in \mathcal{X}}\|\beta(x)\|_2 \le C_3$;
    \item $H(\theta):=\frac{1}{n}\sum_{i=1}^n h_i(\theta)$ is nonsingular and 
    \[\sup_{\theta\in \Theta}\|H(\theta)^{-1}\|_{\textrm{op}} \le C_4;\]
    \item there exists suitable $\Delta > 0$, such that 
    \[\sup_{\|\theta - \htheta_{\vone}\|_2<\Delta} \frac{1}{n} \sum_{i=1}^n \|h_i(\theta) - h_i(\htheta_{\vone})\|_F \le C_5 \|\theta - \htheta_{\vone}\|_2.\]
\end{enumerate}

\end{assumption}

\begin{theorem}[Validity and Cost on Nonlinear Models]
\label{thm:validity-cost-nonlinear}
For any data point $x_0\in \mathcal{X}$, let $\tilde{x}_0$ be the CFE of $x_0$, and let $\tilde{x}_0^{(k)}$ be the solution of the optimization problem~(\ref{eq:krr-cfe-approximate-single-constraint}). Assume the classifiers satisfy Assumption~\ref{assump:regularity}. Then we can properly choose $\delta$ such that, if $\tilde{x}_0^{(k)}$ exists,  $\tilde{x}_0^{(k)}$ remains a valid CFE for all possible removal of $k$ data points, i.e., 
\[f_{\htheta_w}(\tilde{x}_0^{(k)}) \ge 0, \forall w\in \mathcal{W}^{(k)}.\]

Furthermore, the cost of implementing $\tilde{x}_0^{(k)}$ is upper bounded as following,
\begin{align}
    & \|\tilde{x}_0^{(k)} - x_0\|_2 \nonumber \\
    \le &\|\tilde{x}_0 - x_0\|_2 + \min_{\substack{x\in \mathcal{X},\\ f_{\htheta_{\vone}}(x) - f_{\htheta_{\vone}}(\tilde{x}_0) \ge \frac{kC}{n}}} \|x - \tilde{x}_0\|_2, \label{eq:cost-bound-nonlinear}
\end{align}
where $C$ is a constant independent of $n$.

Specially, if $f_{\htheta_{\vone}}$ is $\mu$-strongly convex, then 
\begin{align}
    \|\tilde{x}_0^{(k)} - x_0\|_2 \le \|\tilde{x}_0 - x_0\|_2 + \frac{2kC}{n \mu}. \label{eq:cost-bound-nonlinear-convex}
\end{align}
\end{theorem}

The proof of Theorem~\ref{thm:validity-cost-nonlinear} follows similar steps as Theorem~\ref{thm:validity-cost-linear}. Assumption~\ref{assump:regularity} is specifically baked to bound the difference $|\tilde{f}_{\htheta_w}(x) - f_{\htheta_w}(x)|$ and the difference $|\tilde{f}_{\htheta_w}(x) - f_{\htheta_{\vone}}(x)|$. The detailed proof can be found in Appendix~\ref{app:proof-validity-cost-nonlinear}.

Theorem~\ref{thm:validity-cost-nonlinear} shares similar insights as the linear case while generalizing the results to a broader family of models beyond linear models. Admittedly, the assumptions are relatively strong for them to be held on very complex models such as neural networks. However, in applications where explainability is of major interest, simpler models are often preferred~\citep{Srinivas2022-lr}. So this result still provides valuable insights in practice.
\section{Experimental Evaluation}

\begin{figure*}[t!]
        \centering
        \begin{subfigure}[b]{0.27\textwidth}
            \centering
            \includegraphics[width=\textwidth]{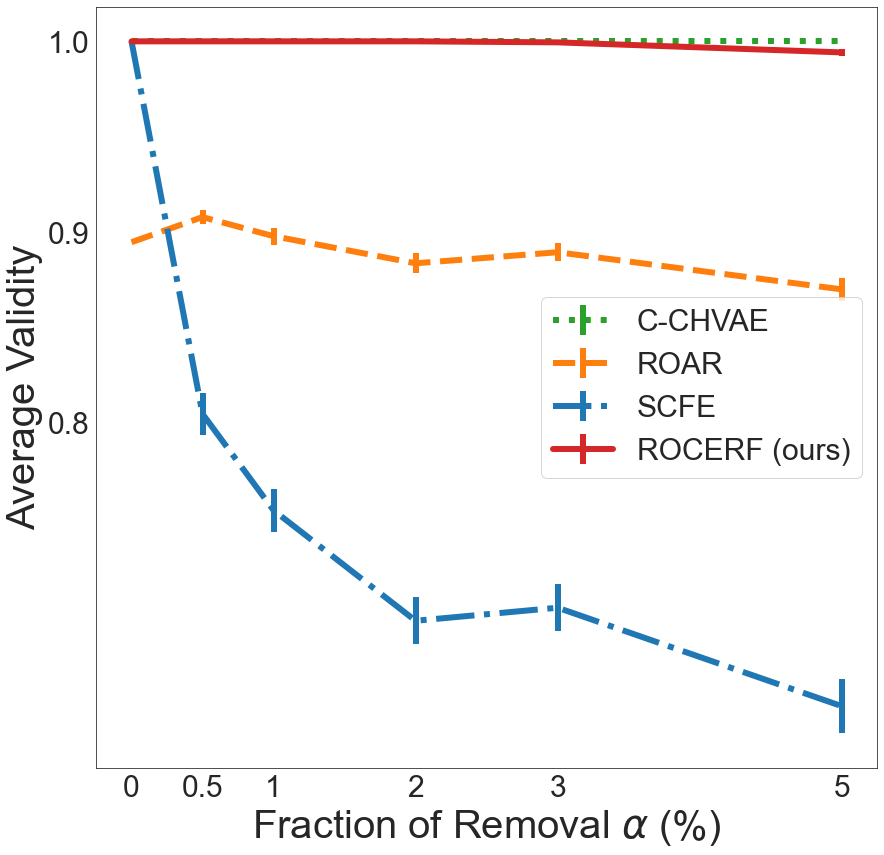}
            \caption{German Credit.}
            \label{fig:german-linear}
        \end{subfigure}
        \begin{subfigure}[b]{0.27\textwidth}
            \centering
            \includegraphics[width=\textwidth]{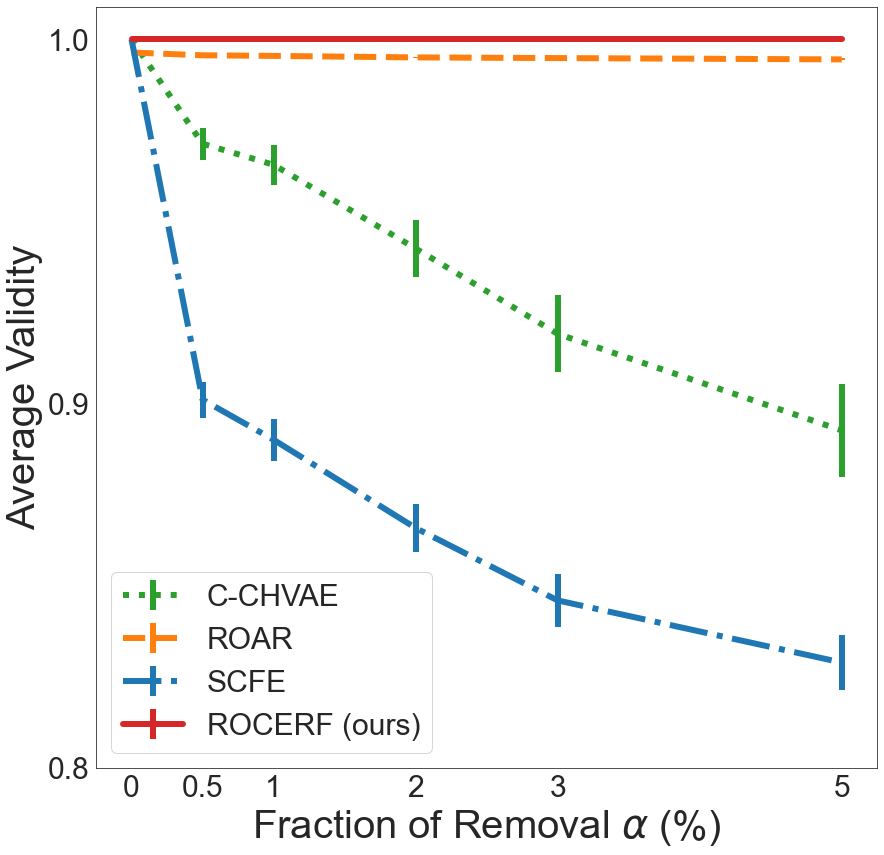}
            \caption{COMPAS.}
            \label{fig:compas-linear}
        \end{subfigure}
        \begin{subfigure}[b]{0.27\textwidth}
            \centering
            \includegraphics[width=\textwidth]{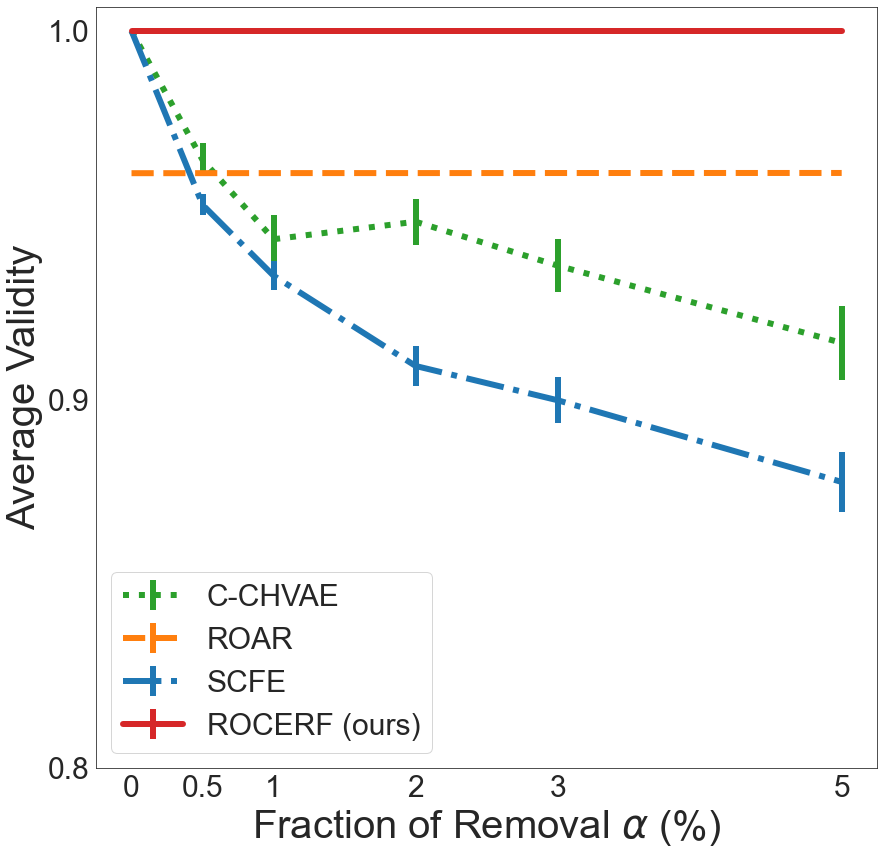}
            \caption{Adult.}
            \label{fig:adult-linear}
        \end{subfigure}
        \caption{Average validity of different counterfactual explanation methods applied to logistic regression models on three datasets. In each figure, the x-axis corresponds to the fraction of data removal $\alpha$ and the y-axis corresponds to the average validity. The error bars indicate the standard errors across $M=100$ trials with each trial having an $\alpha$ fraction of training data points randomly removed.}  %
        \label{fig:validity-linear}
\end{figure*}

\begin{figure*}[t!]
        \centering
        \begin{subfigure}[b]{0.27\textwidth}
            \centering
            \includegraphics[width=\textwidth]{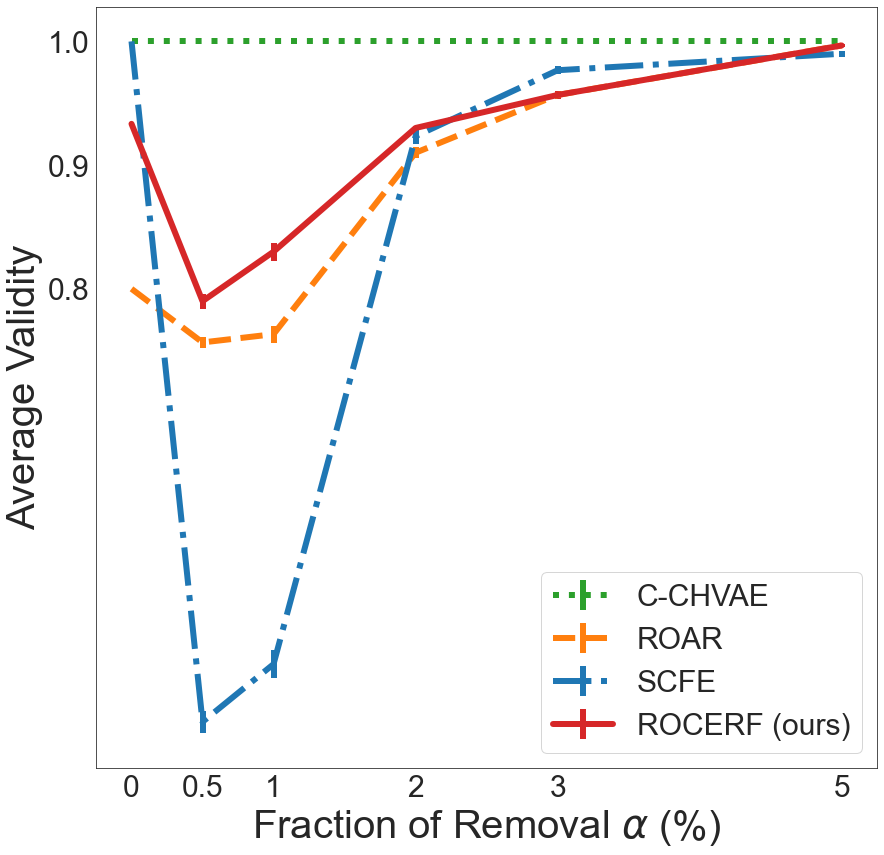}
            \caption{German Credit.}
            \label{fig:german-nonlinear}
        \end{subfigure}
        \begin{subfigure}[b]{0.27\textwidth}
            \centering
            \includegraphics[width=\textwidth]{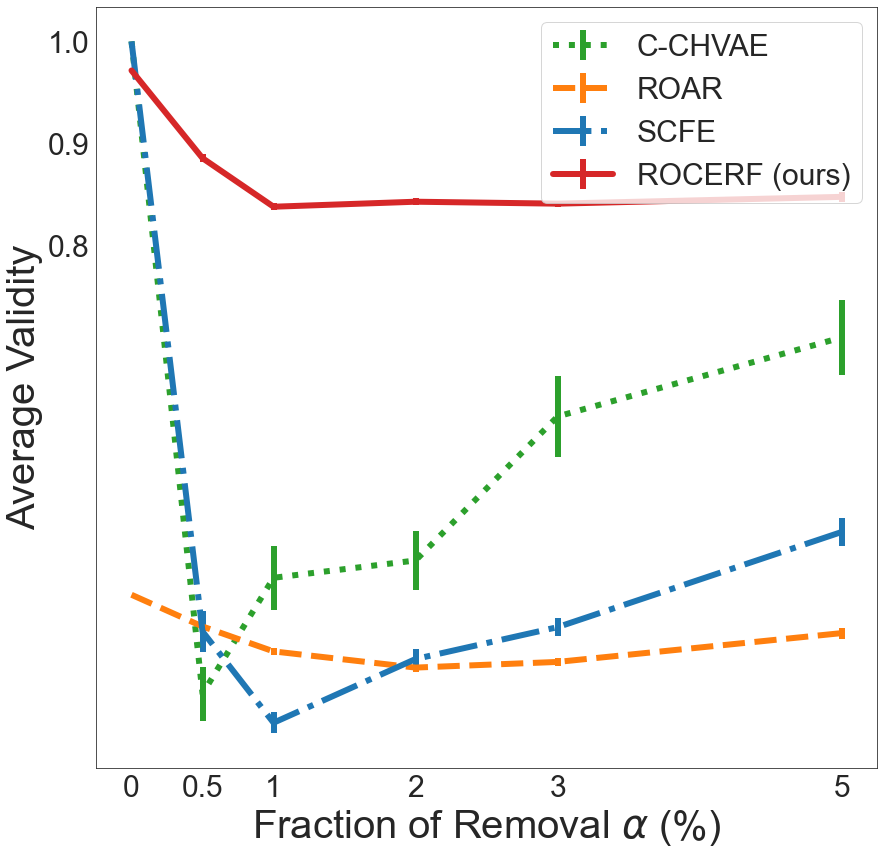}
            \caption{COMPAS.}
            \label{fig:compas-nonlinear}
        \end{subfigure}
        \begin{subfigure}[b]{0.27\textwidth}
            \centering
            \includegraphics[width=\textwidth]{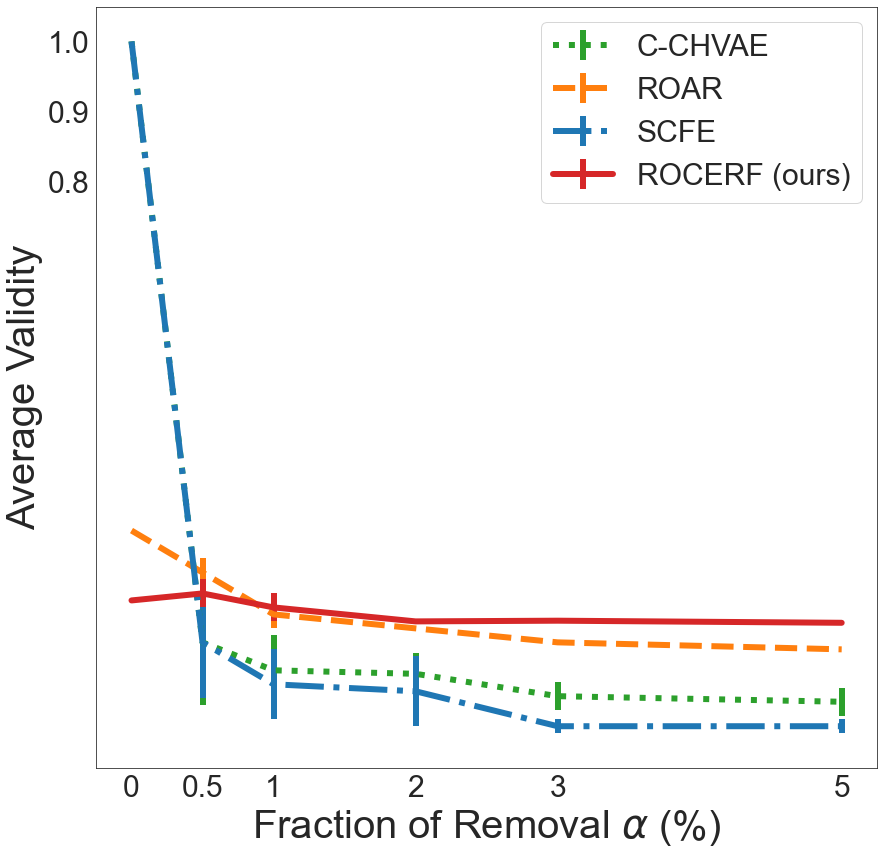}
            \caption{Adult.}
            \label{fig:adult-nonlinear}
        \end{subfigure}
        \caption{Average validity of different counterfactual explanation methods applied to neural network models on three datasets. See Figure~\ref{fig:validity-linear} for more details about the plot setting.}  %
        \label{fig:validity-nonlinear}
\end{figure*}

In this section, we empirically evaluate the validity and cost of the counterfactual explanations output by our framework, and compare them with other state-of-the-art counterfactual explanation methods. We first introduce the general experimental setup and then present experimental results on three real-world datasets with logistic regression and neural network models.

\subsection{Experimental Setup}

We conduct experiments on three real-world datasets that are commonly used to benchmark counterfactual explanation methods. For each dataset, we split the dataset into training, validation, and test sets. We train a machine learning model (which we will refer as the \emph{original model}) using the training set, and select the negative samples (data points classified as $-1$) in the test set. Then we apply different counterfactual explanation methods on these negative test samples to obtain a CFE for each of the sample. We calculate and report the average cost of the CFEs over all the negative test samples. To evaluate the validity under the right to be forgotten, we randomly remove a small fraction, $\alpha$, of the training data points and retrain a new model (which we will refer as the \emph{retrained model}), and then we evaluate the validity over all the negative test samples under the retrained model. We repeat this process of random removal, retrain, and validity evaluation for $M$ times and report the average validity. We fix $M=100$ on all experiments and vary $\alpha \in \{0.5\%, 1\%, 2\%, 3\%, 5\%\}$.

\paragraph{Datasets.} We use three real-world binary classification datasets collected from high-stakes decision making scenarios. 1) \textit{German Credit}~\citep{UCI} comprises of 1000 data points where each data point has 60 features including demographic (age, gender), personal (marital status), and financial (income, credit duration) information of a customer. These data points are labeled as ``good'' or ``bad'' in terms of credit risk. 2) \textit{Adult}~\citep{yeh2009comparisons} contains samples from 48,842 individuals, and each sample contains demographic (e.g., age, race, and gender), education (degree), employment (occupation, hours-per week), personal (marital status, relationship), and financial (capital gain/loss) features. 3) \textit{COMPAS}~\citep{jordan15:effect} comprises of criminal records and demographic features of 18,876 defendants who were released on bail at the US state courts during the period 1990-2009. The prediction target is ``bail'' or ``no bail'' given the defendant's data.

\vspace{-0.1in}
\paragraph{Predictive Models.} We experiment with regularized logistic regression and deep neural networks. For regularized logistic regression, we use the default implementation from the Scikit-Learn package\footnote{\url{https://scikit-learn.org/stable/}.}. For neural networks, we use a 3-layer fully-connected feedforward neural network. Please see Appendix~\ref{app:setup} for more details about the implementation.

\vspace{-0.1in}
\paragraph{Evaluation Metrics.} We evaluate the counterfactual explanation methods in terms of average validity and cost, which are the two most commonly used metrics in the counterfactual explanation literature~\citep{Verma2020-bp}. Denote the set of negative samples under the original model $f_{\htheta_{\vone}}$ as $\mathcal{T}$ and the set of $M$ random removals as $\mathcal{V} \subseteq \mathcal{W}^{(\lceil \alpha n \rceil)}, |\mathcal{V}| = M$. Suppose the CFE of a sample $x$ is denoted as $c(x)$. Then the \emph{average validity} is defined as
\[\frac{1}{M}\sum_{w\in \mathcal{V}} \frac{1}{|\mathcal{T}|}\sum_{x\in \mathcal{T}} \mathbbm{1}[f_{\htheta_w}(c(x)) = 1],\]
where $\mathbbm{1}[\cdot]$ is the indicator function. And the \emph{average cost} is defined as
\vspace{-0.1in}
\[\frac{1}{|\mathcal{T}|}\sum_{x\in \mathcal{T}} \|c(x) - x\|_2.\]
In Appendix~\ref{app:results}, we also report an alternative average cost with the L2 norm being replaced by L1 norm.

\paragraph{Baseline Methods.} We compare the proposed method against three state-of-the-art counterfactual explanation methods, SCFE~\citep{Wachter2017-yk}, C-CHVAE~\citep{pawelczyk2020learning}, and ROAR~\citep{Upadhyay2021-ew}. SCFE uses gradient-based optimization to search for CFEs closest to the input sample, which can be viewed as the solution of the problem~(\ref{eq:cfe}). C-CHVAE is a manifold-based method that searches for CFEs in a latent space. ROAR generates CFEs that are robust to small perturbations in model parameters, which is a strong baseline for the problem of interest in this paper.

\paragraph{Hyperparameters.} For the proposed method, \ours, we set the hyperparameter $k$ as 0.5\% of the training set size and fix $\delta=0$ in all experiments in this section. For SCFE, we use the hyper-parameter setting from \citet{pawelczyk2021connections}. For C-CHVAE, we use the recommendations from \citet{pawelczyk2020learning}. We also use the same hyper-parameter setting for ROAR as suggested in \citet{Upadhyay2021-ew}. We refer the readers to Appendix~\ref{app:setup} for more details.

\subsection{Experimental Results}

\begin{table}[]
\centering
\scalebox{0.8}{
\begin{tabular}{m{2.5cm}|m{2cm}|m{2cm}|m{2cm}m{1cm}m{1cm}m{0.9cm}}
\toprule
\multirow{1}{*}{Methods} & \multicolumn{1}{c|}{German Credit} & \multicolumn{1}{c|}{COMPAS} & \multicolumn{1}{c}{Adult} \\
\midrule
SCFE    & 0.82 $\pm$ 0.12  & 0.78 $\pm$ 0.02  & 1.04 $\pm$ 0.006  \\ 
C-CHVAE &  8.51 $\pm$ 0.38  & 5.93 $\pm$ 0.11 & 3.79 $\pm$ 0.013 \\ 
ROAR    & 1.45 $\pm$ 0.09 & 1.08 $\pm$ 0.01 & 1.07 $\pm$ 0.006 \\ 
\ours (ours) & 1.35 $\pm$ 0.14 & 0.87 $\pm$ 0.02 & 1.14 $\pm$ 0.006  \\
\bottomrule
\end{tabular}
}
\caption{Average cost of different recourse methods applied to logistic regression models on three datasets. The cost is measured in terms of L2 norm.}
\label{tab:cost-linear-l2}
\end{table}

\begin{table}[]
\centering
\scalebox{0.8}{
\begin{tabular}{m{2.5cm}|m{2cm}|m{2cm}|m{2cm}m{1cm}m{1cm}m{0.9cm}}
\toprule
\multirow{1}{*}{Methods} & \multicolumn{1}{c|}{German Credit} & \multicolumn{1}{c|}{COMPAS} & \multicolumn{1}{c}{Adult} \\
\midrule
SCFE    & 1.18 $\pm$ 0.08  & 0.97 $\pm$ 0.11  &  1.00 $\pm$ 0.09  \\ 
C-CHVAE &  4.45 $\pm$ 0.18  & 5.98 $\pm$ 0.12 &  8.83 $\pm$ 0.31 \\ 
ROAR    & 3.84 $\pm$ 0.33 & 1.09 $\pm$ 0.13  & 4.07 $\pm$ 0.55 \\ 
\ours (ours) & 2.76 $\pm$ 0.22 & 3.07 $\pm$ 0.08 & 4.06 $\pm$ 0.52  \\ 
\bottomrule
\end{tabular}
}
\caption{Average cost of different recourse methods applied to neural network models on three datasets. The cost is measured in terms of L2 norm.}
\label{tab:cost-nonlinar-l2}
\end{table}

The experimental results of average validity on logistic regression and neural network models are respectively shown in Figure~\ref{fig:validity-linear} and Figure~\ref{fig:validity-nonlinear}. The results of average cost on logistic regression and neural network models are respectively shown in Table~\ref{tab:cost-linear-l2} and Table~\ref{tab:cost-nonlinar-l2}.

We first look at the results on logistic regression models. As can be seen in Figure~\ref{fig:validity-linear}, the proposed method, \ours, achieves 100\% average validity in almost all experimental settings for logistic regression. This result validates the strong theoretical guarantee on validity stated in Theorem~\ref{thm:validity-cost-linear}. As a comparison, all the baseline methods suffer from significant drops in terms of average validity in some or all experimental settings. 

In terms of the tradeoff between cost (Table~\ref{tab:cost-linear-l2}) and validity (Figure~\ref{fig:validity-linear}), while SCFE always has the lowest cost, it has significantly worse validity than all other methods even for $\alpha = 0.5\%$; C-CHVAE is also inferior to the proposed method as it has both significantly higher costs on all datasets and worse validity on COMPAS and Adult; ROAR is closer to our method but our method consistently outperforms ROAR in terms of validity and has smaller or similar costs than ROAR. Overall, the empirical results both validate our theoretical analysis in Theorem~\ref{thm:validity-cost-linear} and verify that the proposed method outperforms baseline methods.

Next, we look at the results on neural network models. As the change of models after data removals becomes less predictable for these complex models, the performance of counterfactual explanation methods is more dataset dependent. However, we still see that the proposed method is consistently among the best performing methods.

On COMPAS dataset (Figure~\ref{fig:compas-nonlinear}), the proposed method clearly outperforms baseline methods in terms of validity. On Adult dataset (Figure~\ref{fig:adult-nonlinear}), SCFE and C-CHVAE are significantly worse in validity except for on the original model ($\alpha = 0\%$); the proposed method performs similarly as ROAR in terms of both validity and cost.
On German Credit dataset (Figure~\ref{fig:german-nonlinear}), the results of validity seem to be counter-intuitive: the average validity becomes 100\% for all methods after removing a larger fraction of training data. This is possibly because the dataset is small and the decision boundary of the complex models changes dramatically after data removal. In addition, there are only 27 negative test samples under the original neural network model. The dramatical change in decision boundary may make all the test samples suddenly lie in a positive area. Nevertheless, on this dataset, C-CHVAE has the best validity but also with the highest cost. The proposed method has a similar validity as ROAR with a smaller cost. 

\section{Conclusions}
In this work, we make one of the initial attempts at addressing the operational gaps between the right to explanation and the right to be forgotten.
In particular, enforcing the right to be forgotten may invalidate actionable (counterfactual) explanations, which in turn violates the right to explanation. To resolve the tension between these two principles, we propose the first algorithmic framework, ROCERF, which generates counterfactual explanations that are provably robust to model updates triggered as a consequence of data deletion requests. The proposed framework not only enjoys theoretical guarantees on validity and cost, but also outperforms several other state-of-the-art counterfactual explanation methods. Our theoretical and empirical results establish that our framework \ours enables us to simultaneously enforce both the right to explanation as well as the right to be forgotten, thus bridging a critical operational gap between the two regulatory principles.

\bibliography{reference}

\begin{thebibliography}{38}
\providecommand{\natexlab}[1]{#1}
\providecommand{\url}[1]{\texttt{#1}}
\expandafter\ifx\csname urlstyle\endcsname\relax
  \providecommand{\doi}[1]{doi: #1}\else
  \providecommand{\doi}{doi: \begingroup \urlstyle{rm}\Url}\fi

\bibitem[Bourtoule et~al.(2021)Bourtoule, Chandrasekaran, Choquette-Choo, Jia,
  Travers, Zhang, Lie, and Papernot]{Bourtoule2021-tz}
Bourtoule, L., Chandrasekaran, V., Choquette-Choo, C.~A., Jia, H., Travers, A.,
  Zhang, B., Lie, D., and Papernot, N.
\newblock Machine unlearning.
\newblock In \emph{2021 {IEEE} Symposium on Security and Privacy ({SP})}, pp.\
  141--159, May 2021.

\bibitem[Broderick et~al.(2020)Broderick, Giordano, and
  Meager]{Broderick2020-lg}
Broderick, T., Giordano, R., and Meager, R.
\newblock An automatic {Finite-Sample} robustness metric: When can dropping a
  little data make a big difference?
\newblock November 2020.

\bibitem[Cao \& Yang(2015)Cao and Yang]{Cao2015-bi}
Cao, Y. and Yang, J.
\newblock Towards making systems forget with machine unlearning.
\newblock In \emph{2015 {IEEE} Symposium on Security and Privacy}, pp.\
  463--480, May 2015.

\bibitem[CCPA(2018)]{CCPA}
CCPA.
\newblock California consumer privacy act (ccpa), 2018.
\newblock URL \url{https://oag.ca.gov/privacy/ccpa}.

\bibitem[Dandl et~al.(2020)Dandl, Molnar, Binder, and Bischl]{dandl2020multi}
Dandl, S., Molnar, C., Binder, M., and Bischl, B.
\newblock Multi-objective counterfactual explanations.
\newblock In \emph{International Conference on Parallel Problem Solving from
  Nature}, pp.\  448--469. Springer, 2020.

\bibitem[Dua \& Graff(2017)Dua and Graff]{UCI}
Dua, D. and Graff, C.
\newblock {UCI} machine learning repository, 2017.
\newblock URL \url{http://archive.ics.uci.edu/ml}.

\bibitem[Freund(2004)]{Freund2004-in}
Freund, R.~M.
\newblock Penalty and barrier methods for constrained optimization.
\newblock \emph{Lecture Notes, Massachusetts Institute of Technology}, 2004.

\bibitem[Garg et~al.(2020)Garg, Goldwasser, and Vasudevan]{Garg2020-nw}
Garg, S., Goldwasser, S., and Vasudevan, P.~N.
\newblock Formalizing data deletion in the context of the right to be
  forgotten.
\newblock In \emph{Advances in Cryptology -- {EUROCRYPT} 2020}, pp.\  373--402.
  Springer International Publishing, 2020.

\bibitem[GDPR(2016)]{GDPR}
GDPR.
\newblock Regulation (eu) 2016/679 of the european parliament and of the
  council of 27 april 2016 on the protection of natural persons with regard to
  the processing of personal data and on the free movement of such data, and
  repealing directive 95/46/ec (general data protection regulation) (text with
  eea relevance), May 2016.

\bibitem[Ginart et~al.(2019)Ginart, Guan, Valiant, and Zou]{Ginart2019-xh}
Ginart, A., Guan, M., Valiant, G., and Zou, J.~Y.
\newblock Making ai forget you: Data deletion in machine learning.
\newblock \emph{Adv. Neural Inf. Process. Syst.}, 2019.

\bibitem[Giordano et~al.(2019)Giordano, Stephenson, Liu, Jordan, and
  Broderick]{Giordano2019-mj}
Giordano, R., Stephenson, W., Liu, R., Jordan, M., and Broderick, T.
\newblock A swiss army infinitesimal jackknife.
\newblock In Chaudhuri, K. and Sugiyama, M. (eds.), \emph{Proceedings of the
  {Twenty-Second} International Conference on Artificial Intelligence and
  Statistics}, volume~89 of \emph{Proceedings of Machine Learning Research},
  pp.\  1139--1147. PMLR, 2019.

\bibitem[Guo et~al.(2020)Guo, Goldstein, Hannun, and Van
  Der~Maaten]{Guo2020-ml}
Guo, C., Goldstein, T., Hannun, A., and Van Der~Maaten, L.
\newblock Certified data removal from machine learning models.
\newblock In \emph{International Conference on Machine Learning}, pp.\
  3832--3842, 2020.

\bibitem[Izzo et~al.(2021)Izzo, Anne~Smart, Chaudhuri, and Zou]{Izzo2021-sm}
Izzo, Z., Anne~Smart, M., Chaudhuri, K., and Zou, J.
\newblock Approximate data deletion from machine learning models.
\newblock In Banerjee, A. and Fukumizu, K. (eds.), \emph{Proceedings of The
  24th International Conference on Artificial Intelligence and Statistics},
  volume 130 of \emph{Proceedings of Machine Learning Research}, pp.\
  2008--2016. PMLR, 2021.

\bibitem[Jordan \& Freiburger(2015)Jordan and Freiburger]{jordan15:effect}
Jordan, K.~L. and Freiburger, T.~L.
\newblock The effect of race/ethnicity on sentencing: Examining sentence type,
  jail length, and prison length.
\newblock \emph{Journal of Ethnicity in Criminal Justice}, 13\penalty0
  (3):\penalty0 179--196, 2015.
\newblock \doi{10.1080/15377938.2014.984045}.
\newblock URL \url{https://doi.org/10.1080/15377938.2014.984045}.

\bibitem[Karimi et~al.(2020)Karimi, Barthe, Balle, and Valera]{karimi2019model}
Karimi, A.-H., Barthe, G., Balle, B., and Valera, I.
\newblock Model-agnostic counterfactual explanations for consequential
  decisions.
\newblock In \emph{International Conference on Artificial Intelligence and
  Statistics (AISTATS)}, 2020.

\bibitem[Koh \& Liang(2017)Koh and Liang]{Koh2017-vo}
Koh, P.~W. and Liang, P.
\newblock Understanding black-box predictions via influence functions.
\newblock In Precup, D. and Teh, Y.~W. (eds.), \emph{Proceedings of the 34th
  International Conference on Machine Learning}, volume~70 of \emph{Proceedings
  of Machine Learning Research}, pp.\  1885--1894. PMLR, 2017.

\bibitem[Laugel et~al.(2017)Laugel, Lesot, Marsala, Renard, and
  Detyniecki]{laugel2017inverse}
Laugel, T., Lesot, M.-J., Marsala, C., Renard, X., and Detyniecki, M.
\newblock Inverse classification for comparison-based interpretability in
  machine learning.
\newblock \emph{arXiv preprint arXiv:1712.08443}, 2017.

\bibitem[Mahajan et~al.(2019)Mahajan, Tan, and Sharma]{mahajan2019preserving}
Mahajan, D., Tan, C., and Sharma, A.
\newblock Preserving causal constraints in counterfactual explanations for
  machine learning classifiers.
\newblock \emph{arXiv preprint arXiv:1912.03277}, 2019.

\bibitem[Mothilal et~al.(2020)Mothilal, Sharma, and Tan]{mothilal2020fat}
Mothilal, R.~K., Sharma, A., and Tan, C.
\newblock Explaining machine learning classifiers through diverse
  counterfactual explanations.
\newblock In \emph{Proceedings of the Conference on Fairness, Accountability,
  and Transparency (FAT*)}, 2020.

\bibitem[Neel et~al.(2021)Neel, Roth, and Sharifi-Malvajerdi]{Neel2021-iu}
Neel, S., Roth, A., and Sharifi-Malvajerdi, S.
\newblock {Descent-to-Delete}: {Gradient-Based} methods for machine unlearning.
\newblock In Feldman, V., Ligett, K., and Sabato, S. (eds.), \emph{Proceedings
  of the 32nd International Conference on Algorithmic Learning Theory}, volume
  132 of \emph{Proceedings of Machine Learning Research}, pp.\  931--962. PMLR,
  2021.

\bibitem[Paszke et~al.(2017)Paszke, Gross, Chintala, Chanan, Yang, DeVito, Lin,
  Desmaison, Antiga, and Lerer]{Paszke2017-ce}
Paszke, A., Gross, S., Chintala, S., Chanan, G., Yang, E., DeVito, Z., Lin, Z.,
  Desmaison, A., Antiga, L., and Lerer, A.
\newblock Automatic differentiation in {PyTorch}.
\newblock October 2017.

\bibitem[Pawelczyk et~al.(2020)Pawelczyk, Broelemann, and
  Kasneci]{pawelczyk2020learning}
Pawelczyk, M., Broelemann, K., and Kasneci, G.
\newblock Learning model-agnostic counterfactual explanations for tabular data.
\newblock In \emph{Proceedings of The Web Conference 2020}, pp.\  3126--3132,
  2020.

\bibitem[Pawelczyk et~al.(2021)Pawelczyk, Bielawski, Van~den Heuvel, Richter,
  and Kasneci]{pawelczyk2021carla}
Pawelczyk, M., Bielawski, S., Van~den Heuvel, J., Richter, T., and Kasneci, G.
\newblock Carla: A python library to benchmark algorithmic recourse and
  counterfactual explanation algorithms.
\newblock In \emph{Advances in Neural Information Processing Systems (NeurIPS)
  (Benchmark and Datasets Track)}, volume~34, 2021.

\bibitem[Pawelczyk et~al.(2022{\natexlab{a}})Pawelczyk, Agarwal, Joshi,
  Upadhyay, and Lakkaraju]{pawelczyk2021connections}
Pawelczyk, M., Agarwal, C., Joshi, S., Upadhyay, S., and Lakkaraju, H.
\newblock Exploring counterfactual explanations through the lens ofadversarial
  examples: A theoretical and empirical analysis.
\newblock In \emph{International Conference on Artificial Intelligence and
  Statistics (AISTATS)}, 2022{\natexlab{a}}.

\bibitem[Pawelczyk et~al.(2022{\natexlab{b}})Pawelczyk, Leemann, Biega, and
  Kasneci]{Pawelczyk2022-qv}
Pawelczyk, M., Leemann, T., Biega, A., and Kasneci, G.
\newblock On the {Trade-Off} between actionable explanations and the right to
  be forgotten.
\newblock August 2022{\natexlab{b}}.

\bibitem[Rawal \& Lakkaraju(2020)Rawal and Lakkaraju]{rawal2020interpretable}
Rawal, K. and Lakkaraju, H.
\newblock Interpretable and interactive summaries ofactionable recourses.
\newblock In \emph{Advances in Neural Information Processing Systems
  (NeurIPS)}, volume~33, 2020.

\bibitem[Rawal et~al.(2021)Rawal, Kamar, and Lakkaraju]{rawal2021modelshifts}
Rawal, K., Kamar, E., and Lakkaraju, H.
\newblock Algorithmic recourse in the wild: Understanding the impact of data
  and model shifts.
\newblock \emph{arXiv:2012.11788}, 2021.

\bibitem[Ribeiro et~al.(2016)Ribeiro, Singh, and Guestrin]{ribeiro2016should}
Ribeiro, M.~T., Singh, S., and Guestrin, C.
\newblock " why should i trust you?" explaining the predictions of any
  classifier.
\newblock In \emph{Proceedings of the 22nd ACM SIGKDD international conference
  on knowledge discovery and data mining (KDD)}, pp.\  1135--1144, 2016.

\bibitem[Srinivas et~al.(2022)Srinivas, Matoba, Lakkaraju, and
  Fleuret]{Srinivas2022-lr}
Srinivas, S., Matoba, K., Lakkaraju, H., and Fleuret, F.
\newblock Efficient training of {Low-Curvature} neural networks.
\newblock In \emph{Advances in Neural Information Processing Systems}, 2022.

\bibitem[Tolomei et~al.(2017)Tolomei, Silvestri, Haines, and
  Lalmas]{tolomei2017interpretable}
Tolomei, G., Silvestri, F., Haines, A., and Lalmas, M.
\newblock Interpretable predictions of tree-based ensembles via actionable
  feature tweaking.
\newblock In \emph{Proceedings of the ACM SIGKDD International Conference on
  Knowledge Discovery \& Data Mining (KDD)}. ACM, 2017.

\bibitem[Upadhyay et~al.(2021)Upadhyay, Joshi, and Lakkaraju]{Upadhyay2021-ew}
Upadhyay, S., Joshi, S., and Lakkaraju, H.
\newblock Towards robust and reliable algorithmic recourse.
\newblock In Ranzato, M., Beygelzimer, A., Dauphin, Y., Liang, P.~S., and
  Vaughan, J.~W. (eds.), \emph{Advances in Neural Information Processing
  Systems}, volume~34, pp.\  16926--16937. Curran Associates, Inc., 2021.

\bibitem[Ustun et~al.(2019)Ustun, Spangher, and Liu]{Ustun2019ActionableRI}
Ustun, B., Spangher, A., and Liu, Y.
\newblock Actionable recourse in linear classification.
\newblock In \emph{Proceedings of the Conference on Fairness, Accountability,
  and Transparency (FAT*)}, 2019.

\bibitem[Van~Looveren \& Klaise(2019)Van~Looveren and
  Klaise]{van2019interpretable}
Van~Looveren, A. and Klaise, J.
\newblock Interpretable counterfactual explanations guided by prototypes.
\newblock \emph{arXiv preprint arXiv:1907.02584}, 2019.

\bibitem[Verma et~al.(2020{\natexlab{a}})Verma, Dickerson, and
  Hines]{Verma2020-bp}
Verma, S., Dickerson, J., and Hines, K.
\newblock Counterfactual explanations for machine learning: A review.
\newblock October 2020{\natexlab{a}}.

\bibitem[Verma et~al.(2020{\natexlab{b}})Verma, Dickerson, and
  Hines]{verma2020counterfactual}
Verma, S., Dickerson, J., and Hines, K.
\newblock Counterfactual explanations for machine learning: A review.
\newblock \emph{arXiv:2010.10596}, 2020{\natexlab{b}}.

\bibitem[Wachter et~al.(2017)Wachter, Mittelstadt, and Russell]{Wachter2017-yk}
Wachter, S., Mittelstadt, B., and Russell, C.
\newblock Counterfactual explanations without opening the black box: Automated
  decisions and the {GDPR}.
\newblock \emph{SSRN Electron. J.}, 2017.

\bibitem[Wachter et~al.(2018)Wachter, Mittelstadt, and
  Russell]{wachter2017counterfactual}
Wachter, S., Mittelstadt, B., and Russell, C.
\newblock Counterfactual explanations without opening the black box: automated
  decisions and the gdpr.
\newblock \emph{Harvard Journal of Law \& Technology}, 31\penalty0 (2), 2018.

\bibitem[Yeh \& Lien(2009)Yeh and Lien]{yeh2009comparisons}
Yeh, I.-C. and Lien, C.-h.
\newblock The comparisons of data mining techniques for the predictive accuracy
  of probability of default of credit card clients.
\newblock In \emph{Expert Systems with Applications}, 2009.

\end{thebibliography}
\bibliographystyle{icml2023}

\newpage
\appendix
\section{The Optimization Algorithm}
\label{app:algorithm}

\begin{algorithm}
\caption{ROCERF.}
\label{alg:opt}
\textbf{Input}: $x_0, f_{\mathcal{A}}^{(k)}, \delta, \mathcal{X}, T$.\\
\textbf{Output}: $\tilde{x}_0^{(k)}$.
\begin{algorithmic}[1]  
\STATE Set $\lambda = 0.1$;
\STATE Set $x' = \arg\min_{x\in \mathcal{X}} \lambda\phi(\delta - f_{\mathcal{A}}^{(k)}(x)) + \|x - x_0\|_2$;
\STATE \texttt{// Find initial left value $\lambda$}
\WHILE{$f_{\mathcal{A}}^{(k)}(x') \ge \delta$}
    \STATE Set $\lambda = \lambda / 2$;
    \STATE Set $x' = \arg\min_{x\in \mathcal{X}} \lambda\phi(\delta - f_{\mathcal{A}}^{(k)}(x)) + \|x - x_0\|_2$;
\ENDWHILE
\STATE Set $\lambda' = \lambda$;
\STATE \texttt{// Find initial right value $\lambda'$}
\WHILE{$f_{\mathcal{A}}^{(k)}(x') < \delta$}
    \STATE Set $\lambda' = \lambda' \times 2$;
    \STATE Set $x' = \arg\min_{x\in \mathcal{X}} \lambda\phi(\delta - f_{\mathcal{A}}^{(k)}(x)) + \|x - x_0\|_2$;
\ENDWHILE
\STATE \texttt{// Binary search between $\lambda$ and $\lambda'$}
\FOR{$t=1, 2, \ldots, T$}
    \STATE Set $\lambda_t = (\lambda + \lambda') / 2$;
    \STATE Set $x_t = \arg\min_{x\in \mathcal{X}} \lambda_t\phi(\delta - f_{\mathcal{A}}^{(k)}(x)) + \|x - x_0\|_2$;
    \IF{$f_{\mathcal{A}}^{(k)}(x_t) < \delta$}
        \STATE Set $\lambda = \lambda_t$;
    \ELSE
        \STATE Set $\lambda' = \lambda_t$;
    \ENDIF
\ENDFOR
\STATE Set $\tilde{x}_0^{(k)} = \arg\min_{x\in \mathcal{X}} \lambda\phi(\delta - f_{\mathcal{A}}^{(k)}(x)) + \|x - x_0\|_2$;
\STATE \textbf{Return} $\tilde{x}_0^{(k)}$;
\end{algorithmic}
\end{algorithm}

\section{Proof of Theorem~\ref{thm:validity-cost-linear}}
\label{app:proof-validity-cost-linear}

\subsection{Lemmas}
We start by introducing a few useful lemmas.

For any $w\in \mathcal{W}^{(k)}$, define the following LKO estimator of $\htheta_{w}$:
\begin{equation}
    \label{eq:approx-theta}
    \ttheta_w := \htheta_{\vone} + H^{-1}\left(\frac{1}{n}\sum_{i: w_i=0} g_i(\htheta_{\vone})\right).
\end{equation}
Note that for linear models, Eq.~(\ref{eq:taylor-explicit-form-linear}) can be rewritten as $\tilde{f}_{\htheta_w}(x) = \ttheta_w^T x$.

The difference between $\ttheta_{w}$ and $\htheta_{w}$ can be bounded by the following lemma.
\begin{lemma}[Corollary 1 in \citet{Giordano2019-mj}]
\label{lemma:lko-approx-error}
Let $H(\theta) = \frac{1}{n}\sum_{i=1}^n h_i(\theta)$. Assume the following quantities are bounded by constants independent of $n$: (1) $\sup_{\theta\in \Theta} \|H(\theta)^{-1}\|_{\textrm{op}}$; (2) $\frac{1}{n}\sum_{i=1}^n \|g_i(\theta)\|_2^2$; (3) $\frac{1}{n}\sum_{i=1}^n \|h_i(\theta)\|_F^2$. Also assume that there exists a suitable $\Delta > 0$, such that the following quantity is bounded by a constant independent of $n$: $\sup_{\|\theta - \htheta_{\vone}\|_2<\Delta} \frac{1}{n} \sum_{i=1}^n \|h_i(\theta) - h_i(\htheta_{\vone})\|_F / \|\theta - \htheta_{\vone}\|_2$. 
Then for any small integer $k$, there exists a constant $C_1$ independent of $n$, such that
\begin{align}
    \sup_{w\in \mathcal{W}^{(k)}} \|\ttheta_w - \htheta_w\|_2 \le \frac{k C_1}{n}. \label{eq:lko-approx-error}
\end{align}
\end{lemma}

We also have the following results for strongly convex models.
\begin{lemma}[Lemma 8 in \citet{Neel2021-iu}]
    \label{lemma:strong-convexity}
    Suppose $l:\Theta \rightarrow \sR$ is $\mu$-strongly convex and let $\theta^* = \arg\min_{\theta\in \Theta}l(\theta)$. We have that for any $\theta\in \Theta$, $l(\theta) \ge l(\theta^*) + \frac{\mu}{2}\|\theta - \theta^*\|_2^2$.
\end{lemma}

\begin{lemma}
\label{lemma:theta-diff}
Assume $l_i, i=1,\ldots, n$ are $L$-Lipschitz and $\mu$-strongly convex. For a fixed positive integer $k$, there exists a constant $C_2$ independent of $n$, such that for any $w\in \mathcal{W}^{k}$,
\begin{align}
    \|\ttheta_w - \htheta_{\vone}\|_2 \le \frac{k C_2}{n}.
\end{align}
\end{lemma}

\begin{proof}[Proof of Lemma~\ref{lemma:theta-diff}]

We bound $\|\ttheta_w - \htheta_{\vone}\|_2$ by the summation of $\|\ttheta_w - \htheta_w\|_2$ and $\|\htheta_w - \htheta_{\vone}\|_2$. From Lemma~\ref{lemma:lko-approx-error}, we already have $\|\ttheta_w - \htheta_w\|_2 \le \frac{k C_1}{n}$. We now bound $\|\htheta_w - \htheta_{\vone}\|_2$ largely following the proof of Lemma 8 (Sensitivity) in \citet{Neel2021-iu}.

WLOG, assume the the first $k$ data points are removed in $w$, i.e., $w_1 = w_2 = \ldots = w_k = 0$ while $w_{k+1} = \ldots = w_n = 1$. Then we have

\begin{align}
    \frac{1}{n}\sum_{i=1}^n l_i(\htheta_w) &= \frac{n-k}{n} \frac{1}{n-k}\sum_{i=k+1}^n l_i(\htheta_w) + \frac{1}{n}\sum_{i=1}^k l_i(\htheta_w) \nonumber \\
    &\le \frac{n-k}{n} \frac{1}{n-k}\sum_{i=k+1}^n l_i(\htheta_{\vone}) + \frac{1}{n}\sum_{i=1}^k l_i(\htheta_w) \label{eq:minimizer} \\
    &= \frac{1}{n}\sum_{i=k+1}^n l_i(\htheta_{\vone}) + \frac{1}{n}\sum_{i=1}^k \left(l_i(\htheta_w) - l_i(\htheta_{\vone})\right) \nonumber \\
    &\le \frac{1}{n}\sum_{i=k+1}^n l_i(\htheta_{\vone}) + \frac{kL}{n}\|\htheta_w - \htheta_{\vone}\|_2, \label{eq:lipschitz}
\end{align}
where (\ref{eq:minimizer}) is because $\htheta_w$ is the minimizer of $\frac{1}{n-k}\sum_{i=k+1}^n l_i(\theta)$, while in (\ref{eq:lipschitz}) we have utilized the fact that each $l_i$ is $L$-Lipschitz.

On the other hand, by Lemma~\ref{lemma:strong-convexity}, we have
\[\frac{1}{n}\sum_{i=1}^n l_i(\htheta_w) \ge \frac{1}{n}\sum_{i=k+1}^n l_i(\htheta_{\vone}) + \frac{\mu}{2} \|\htheta_w - \htheta_{\vone}\|_2^2.\]

Combining the two inequalities above, we have $\|\htheta_w - \htheta_{\vone}\|_2 \le \frac{k(2L/\mu)}{n}$.

Therefore, letting $C_2 = C_1 + \frac{2L}{\mu}$, we have

\[\|\ttheta_w - \htheta_{\vone}\|_2 \le \|\ttheta_w - \htheta_w\|_2 + \|\htheta_w - \htheta_{\vone}\|_2 \le \frac{kC_2}{n}.\]

\end{proof}

\subsection{Useful Facts of Regularized Logistic Regression}
Define $\sigma(x; \theta) = \frac{1}{1+\exp(-\theta^T x)}$. For regularized logistic regression with the loss defined as $l_i(\theta) = \log(1 + \exp(-y_i \theta^T x_i)) + \gamma \|\theta\|_2^2$, we have
\begin{align}
    g_i(\theta) &= -\frac{1}{1 + \exp(y_i\theta^T x_i)}(y_i x_i) + \gamma \theta, \label{eq:LR-g}\\
    h_i(\theta) &= \sigma(x_i; \theta) (1 - \sigma(x_i; \theta)) x_i x_i^T + \gamma I. \label{eq:LR-h}
\end{align}

We can verify that regularized logistic regression satisfies all the assumptions in Lemma~\ref{lemma:lko-approx-error}. First, we know that the eigen value of the Hessian is lower-bounded by $\gamma$, so the eigen value of the inverse Hessian is upper bounded by $1/\gamma$. Hence $\sup_{\theta\in \Theta} \|H(\theta)^{-1}\|_{\textrm{op}}$ is bounded. Next, under the assumption that both the feature vector and model parameters have bounded norm, it is easy to show that $\|g_i(\theta)\|_2$ and $\|h_i(\theta)\|_F$ are bounded from Eq.~(\ref{eq:LR-g}) and Eq.~(\ref{eq:LR-h}). Hence both $\frac{1}{n}\sum_{i=1}^n \|g_i(\theta)\|_2^2$ and $\frac{1}{n}\sum_{i=1}^n \|h_i(\theta)\|_F^2$ are bounded. Finally, $h_i(\theta)$ is Lipschitz continuous so the last assumption is also verified.

We can also verify that $l_i$ are Lipschitz and strongly convex so the regularized logistic regression satisfies the assumptions in Lemma~\ref{lemma:theta-diff}.

\subsection{Proof of Theorem~\ref{thm:validity-cost-linear}}

Next, we are ready to prove Theorem~\ref{thm:validity-cost-linear}.

\begin{proof}[Proof of Theorem~\ref{thm:validity-cost-linear}]

The validity of $\tilde{x}_0^{(k)}$ holds if for any $w\in \mathcal{W}^{(k)}$ and $x\in \mathcal{X}$, $\tilde{f}_{\htheta_w}(x) \ge \delta$ implies $f_{\htheta_w}(x) \ge 0$. Now we investigate the choice of $\delta$ that guarantees the above condition holds while not being too large.

Under the linear model assumption, for any $w\in \mathcal{W}^{(k)}$, we have
\begin{align*}
    f_{\htheta_w}(x) &= \htheta_w^T x \\
    &= \ttheta_w^T x + (\htheta_w - \ttheta_w)^T x \\
    &\ge \tilde{f}_{\htheta_w}(x) - \|\htheta_w - \ttheta_w\|_2 \|x\|_2 \\
    &\ge \tilde{f}_{\htheta_w}(x) - \frac{kC_1}{n}\cdot B,
\end{align*}
where recall that $B=\sup_{x\in \mathcal{X}}\|x\|_2$.

If we set $\delta = \frac{kC_1 B}{n}$, then $\tilde{f}_{\htheta_w}(x) \ge \delta$ implies $f_{\htheta_w}(x) \ge 0$, in which case the validity of $\tilde{x}_0^{(k)}$ is guaranteed.

For the cost, as $\tilde{x}_0^{(k)}$ is the minimizer of the problem~(\ref{eq:krr-cfe-approximate}), we have $\|\tilde{x}_0^{(k)} - x_0\|_2 \le \|x - x_0\|_2$ for any $x$ in the feasible set of the the problem~(\ref{eq:krr-cfe-approximate}). Furthermore, for any $x$, $\|x - x_0\|_2 \le \|\tilde{x}_0 - x_0\| + \|x - \tilde{x}_0\|$. So we only need to focus on the bound of $\|x - \tilde{x}_0\|_2$ for some $x$ in the feasible set.

To begin with, we make the following transformation of $\tilde{f}_{\htheta_w}(x)$.
\begin{align*}
    \tilde{f}_{\htheta_w}(x) &= \ttheta_w^T x \\
     &= \htheta_{\vone}^T \tilde{x}_0 + \htheta_{\vone}^T (x - \tilde{x}_0) + (\ttheta_w - \htheta_{\vone})^T x \\
     &\ge \htheta_{\vone}^T (x - \tilde{x}_0) + (\ttheta_w - \htheta_{\vone})^T x && \text{$\htheta_{\vone}^T \tilde{x}_0 \ge 0$ by Definition~\ref{def:cfe}} \\
     &\ge \htheta_{\vone}^T (x - \tilde{x}_0) - \|\ttheta_w - \htheta_{\vone}\|_2 \|x\|_2 \\
     &\ge \htheta_{\vone}^T (x - \tilde{x}_0) - \frac{kC_2 B}{n}. && \text{Lemma~\ref{lemma:theta-diff}}
\end{align*}
For $x$ to be in a feasible set, it suffices to have $\tilde{f}_{\htheta_w}(x) \ge \delta$ for all $w\in \mathcal{W}^{(k)}$. Set $\delta = \frac{kC_1 B}{n}$ and let $x' = \tilde{x}_0 + \frac{kC}{n\|\htheta_{\vone}\|_2^2}\htheta_{\vone}$, where $C:= (C_1 + C_2)B$. Then for any $w$, we have
\begin{align*}
    \tilde{f}_{\htheta_w}(x') - \delta &\ge \htheta_{\vone}^T (x' - \tilde{x}_0) - \frac{kC_2 B}{n} - \frac{kC_1 B}{n} = 0.
\end{align*}

So $x'$ is in the feasible set. Therefore,
\[\|\tilde{x}_0^{(k)} - x_0\|_2 \le \|x' - x_0\|_2 \le \|\tilde{x}_0 - x_0\|_2 + \frac{kC}{n\|\htheta_{\vone}\|_2}.\]

\end{proof}

\section{Proof of Theorem~\ref{thm:validity-cost-nonlinear}}
\label{app:proof-validity-cost-nonlinear}

We first apply a result from \citet{Broderick2020-lg} to bound the difference $|\tilde{f}_{\htheta_w}(x) - f_{\htheta_w}(x)|$.

\begin{lemma}[Direct Application of Theorem 1 in \citet{Broderick2020-lg}]
\label{lemma:fw-diff}
Under Assumption~\ref{assump:regularity}, there exists a constant $C_f$ independent of $n$, such that
\[\sup_{x\in \mathcal{X}, w\in \mathcal{W}^{(k)}} |\tilde{f}_{\htheta_w}(x) - f_{\htheta_w}(x)| < \frac{kC_f}{n}.\]
\end{lemma}

Next, we use this result to prove Theorem~\ref{thm:validity-cost-nonlinear}.
\begin{proof}[Proof of Theorem~\ref{thm:validity-cost-nonlinear}]
Similarly as Theorem~\ref{thm:validity-cost-linear}, the validity of $\tilde{x}_0^{(k)}$ holds if for any $w\in \mathcal{W}^{(k)}$ and $x\in \mathcal{X}$, $\tilde{f}_{\htheta_w}(x) \ge \delta$ implies $f_{\htheta_w}(x) \ge 0$. By Lemma~\ref{lemma:fw-diff}, we know that setting $\delta = \frac{kC_f}{n}$ suffices to guarantee the validity.

For the cost, similarly as Theorem~\ref{thm:validity-cost-linear}, we only need to bound $\|x - \tilde{x}_0\|_2$ for some feasible $x$ in problem~\ref{eq:krr-cfe-approximate}.

For any $x\in \mathcal{X}$ and $w\in \mathcal{W}^{(k)}$, we have
\begin{align*}
    \tilde{f}_{\htheta_w}(x) &= f_{\htheta_{\vone}}(\tilde{x}_0) + f_{\htheta_{\vone}}(x) - f_{\htheta_{\vone}}(\tilde{x}_0) + \tilde{f}_{\htheta_w}(x) - f_{\htheta_{\vone}}(x) \\
     &\ge f_{\htheta_{\vone}}(x) - f_{\htheta_{\vone}}(\tilde{x}_0) + \tilde{f}_{\htheta_w}(x) - f_{\theta_{\vone}}(x) && \text{$f_{\htheta_{\vone}}(\tilde{x}_0) \ge 0$ by Definition~\ref{def:cfe}} \\
     &= f_{\htheta_{\vone}}(x) - f_{\htheta_{\vone}}(\tilde{x}_0) + \frac{1}{n} \sum_{i:w_i=0} \beta(x)^T  H^{-1} g_i(\htheta_{\vone}) && \text{Eq.~(\ref{eq:taylor-explicit-form})} \\
     &\ge f_{\htheta_{\vone}}(x) - f_{\htheta_{\vone}}(\tilde{x}_0) - \frac{k C_2 C_3 C_4}{n}. && \text{Assumption~\ref{assump:regularity}}
\end{align*}

For an $x$ to be feasible, it needs to satisfy $\tilde{f}_{\htheta_w}(x) \ge \delta$ for all $w\in \mathcal{W}^{(k)}$. Set $\delta = \frac{kC_f}{n}$ and let 
\[x' = \arg\min_{\substack{x\in \mathcal{X}, \\f_{\htheta_{\vone}}(x) - f_{\htheta_{\vone}}(\tilde{x}_0) \ge \frac{kC}{n}}} \|x - \tilde{x}_0\|_2,\]
where $C:= C_f + C_2 C_3 C_4$. Then for any $w$, we have
\[\tilde{f}_{\htheta_w}(x') - \delta \ge f_{\htheta_{\vone}}(x') - f_{\htheta_{\vone}}(\tilde{x}_0) - \frac{kC}{n} \ge 0.\]
So $x'$ is in the feasible set. Therefore,
\[\|\tilde{x}_0^{(k)} - x_0\|_2 \le \|x' - x_0\|_2 \le \|\tilde{x}_0 - x_0\|_2 + \min_{\substack{x\in \mathcal{X}, \\f_{\htheta_{\vone}}(x) - f_{\htheta_{\vone}}(\tilde{x}_0) \ge \frac{kC}{n}}} \|x - \tilde{x}_0\|_2.\]

Finally, if $f_{\htheta_{\vone}}$ is $\mu$-strongly convex, then for any $z\in \mathcal{X}$,
\[f_{\htheta_{\vone}}(z) \ge f_{\htheta_{\vone}}(\tilde{x}_0) + \frac{\partial f_{\htheta_{\vone}}(x)}{\partial x}\eval{\tilde{x}_0} (z - \tilde{x}_0) + \frac{\mu}{2}\|z - \tilde{x}_0\|_2.\]

Denote $v = \left(\frac{\partial f_{\htheta_{\vone}}(x)}{\partial x}\eval{\tilde{x}_0}\right)^T$. Let $z = \tilde{x}_0 + \frac{2kC}{n\mu}\frac{v}{\|v\|_2}$, then 
\[f_{\htheta_{\vone}}(z) - f_{\htheta_{\vone}}(\tilde{x}_0) \ge \frac{kC}{n}.\]
So $z$ is in the feasible set. Therefore, 
\[\|\tilde{x}_0^{(k)} - x_0\|_2 \le \|x' - x_0\|_2 \le \|\tilde{x}_0 - x_0\|_2 + \frac{2kC}{n\mu}.\]
\end{proof}

\section{Experiment Details}
\subsection{More Detailed Experimental Setup}
\label{app:setup}

\paragraph{Model Training. } We train two models for our experiments : (1) Logistic Regression (LR), and (2) Neural Network (NN). For NN, we have three intermediate layers with twice the number of input nodes for each intermediate layer. We apply centered-softplus activation~\citep{Srinivas2022-lr} for each intermediate layer output. The training procedure involved minimizing the standard cross entropy loss using stochastic gradient descent with 0.01 as the learning rate. The accuracy achieved after training for all the datasets is shown in Table~\ref{tab:acc_performance}. 

\begin{table}[h]
    \centering\small
    \renewcommand{\arraystretch}{0.9}
    \setlength{\tabcolsep}{1.5pt}
    \caption{
     The accuracy of LR and ANN models trained on the datasets.
    }
    {\begin{tabular}{lcc}
    \toprule
    {Dataset} & {LR} & {NN} \\
    \toprule
    \begin{tabular}[l]{@{}l@{}}{German~Credit}\\{COMPAS}\\{Adult}
    \end{tabular} &
    \begin{tabular}[c]{@{}c@{}}
    {72.2\%}\\
    {85.8\%}\\
    {84.0\%}\\
    \end{tabular} &
    \begin{tabular}[c]{@{}c@{}}
    73.9\%\\
    85.1\% \\
    84.7\%\\
    \end{tabular} \\
    \bottomrule
    \end{tabular}}
    \label{tab:acc_performance}
\end{table}

\paragraph{Recourse Method Hyperparameters. } We use default hyper-parameter setting for most baseline methods aligned with authors' guidelines. Specifically, we use step size = 0.05 with a sample size of 1000 per iteration for C-CHVAE, $\delta_{max} = 0.1$ for ROAR.

\paragraph{LIME Approximation of Neural Network Models.} For neural network, we learn a local linear approximation of the model using the perturbation-based framework in LIME \cite{ribeiro2016should}. Specifically, we train a logistic regression model on 10,000 perturbations sampled from $\mathcal{N}(0,\,0.1)$ around the input sample.

\subsection{Additional Results}
\label{app:results}

\paragraph{Average cost in terms of L1 norm.} We provide the L1-norm based average cost of different recourse methods in Table~\ref{tab:cost-linear-l1} and Table~\ref{tab:cost-nonlinear-l1}, respectively for logistic regression and neural network models. The relative trend is almost the same as the results of L2-norm based average cost reported in the main paper.

\begin{table*}[h]
\centering
\small
\begin{tabular}{m{2.5cm}|m{2cm}|m{2cm}|m{2cm}m{1cm}m{1cm}m{0.9cm}}
\toprule
\multirow{1}{*}{Methods} & \multicolumn{1}{c|}{German Credit} & \multicolumn{1}{c|}{COMPAS} & \multicolumn{1}{c}{Adult} \\
\midrule
SCFE    & 5.76 $\pm$ 0.92  & 1.91 $\pm$ 0.06  & 2.98 $\pm$ 0.01   \\ 
C-CHVAE & 48.04 $\pm$ 1.83  & 11.71 $\pm$ 0.23 &  10.37 $\pm$ 0.03   \\ 
ROAR    & 9.61 $\pm$ 0.61 & 2.47 $\pm$ 0.04  & 2.61 $\pm$ 0.01 \\ 
\ours (ours) & 9.44 $\pm$ 1.08 & 2.13 $\pm$ 0.05 & 3.33 $\pm$ 0.02 \\ 
\bottomrule

\end{tabular}
\caption{Average cost of different recourse methods applied to logistic regression models on three datasets. The cost is measured in terms of L1 norm.}
\label{tab:cost-linear-l1}
\end{table*}

\begin{table*}[h]
\centering
\small
\begin{tabular}{m{2.5cm}|m{2cm}|m{2cm}|m{2cm}m{1cm}m{1cm}m{0.9cm}}
\toprule
\multirow{1}{*}{Methods} & \multicolumn{1}{c|}{German Credit} & \multicolumn{1}{c|}{COMPAS} & \multicolumn{1}{c}{Adult} \\
\midrule
SCFE    & 2.97 $\pm$ 1.14  & 2.01 $\pm$ 0.23  & 7.03 $\pm$ 0.87  \\
C-CHVAE & 12.56 $\pm$ 0.53 & 11.31$\pm$ 0.19 & 49.45 $\pm$ 1.88  \\  
ROAR    & 13.46 $\pm$ 0.25 & 2.25 $\pm$ 0.25  & 23.04 $\pm$ 3.68 \\ 
\ours (ours) & 7.85 $\pm$ 0.58 & 6.35 $\pm$ 0.14 & 19.61 $\pm$  3.07 \\ 
\bottomrule
\end{tabular}
\caption{Average cost of different recourse methods applied to neural network models on three datasets. The cost is measured in terms of L2 norm.}
\label{tab:cost-nonlinear-l1}
\end{table*}

\paragraph{Sensitivity analysis with respect to the hyperparameter $k$.} We also conduct a sensitivity analysis on different variants of our method with respect to the hyperparameter $k$. Figure~\ref{fig:validity-linear-sensitivity} shows the results on logistic regression models. All variants of $k$ achieves 100\% validity COMPAS and Adult. On German Credit, the variant with the lowest $k$ has a slight drop in validity for higher fraction of removal, which is fixed by for variants of higher $k$ values. Notably, $k=0.005 n$ corresponds to $\alpha=0.5\%$ and similarly for other values of $k$ and $\alpha$. The value of $k$ refers to the hyperparameter of our method while the value of $\alpha$ refers to the actual fraction of data removal in the evaluation. In reality, our selection of hyperparameter $k$ may not exactly match the actual fraction of removals in the future. However, we note that, in Figure~\ref{fig:validity-linear-sensitivity}, the variants of our method always achieve 100\% validity for $\alpha n$ below the hyperparamter $k$, e.g., $k=0.01 n$ achieves 100\% validity for any $\alpha \le 1\%$ and $k=0.02 n$ achieves 100\% validity for any $\alpha \le 2\%$.

\begin{figure*}[t!]
        \centering
        \begin{subfigure}[b]{0.31\textwidth}
            \centering
            \includegraphics[width=\textwidth]{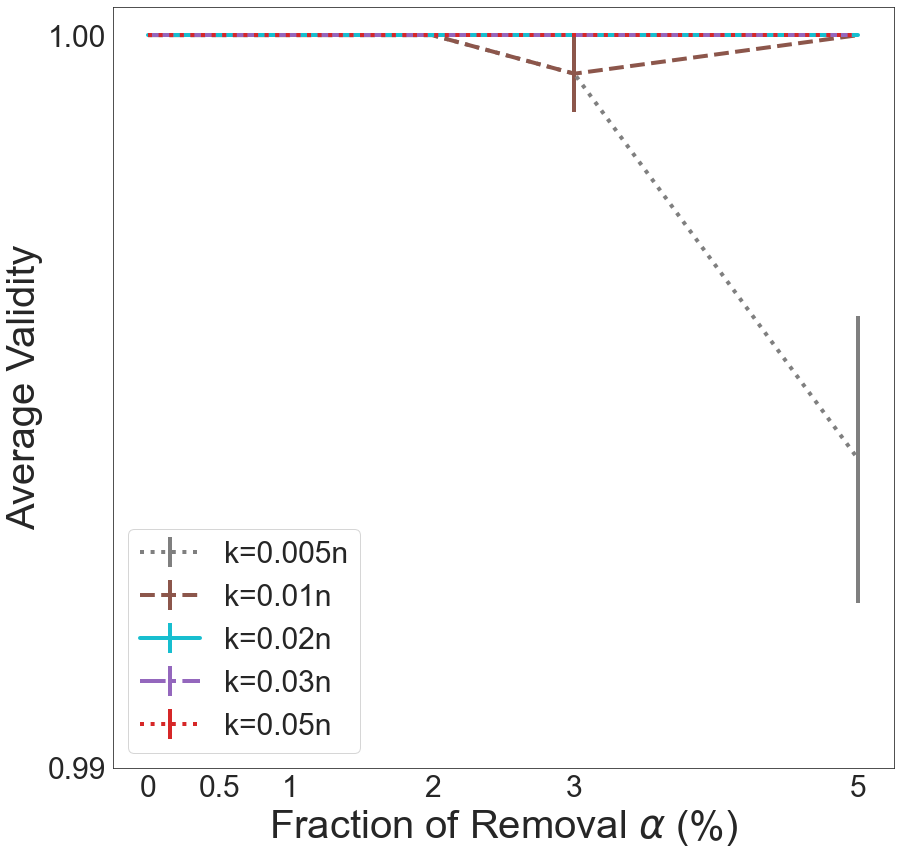}
            \caption{German Credit.}
            \label{fig:german-linear}
        \end{subfigure}
        \begin{subfigure}[b]{0.31\textwidth}
            \centering
            \includegraphics[width=\textwidth]{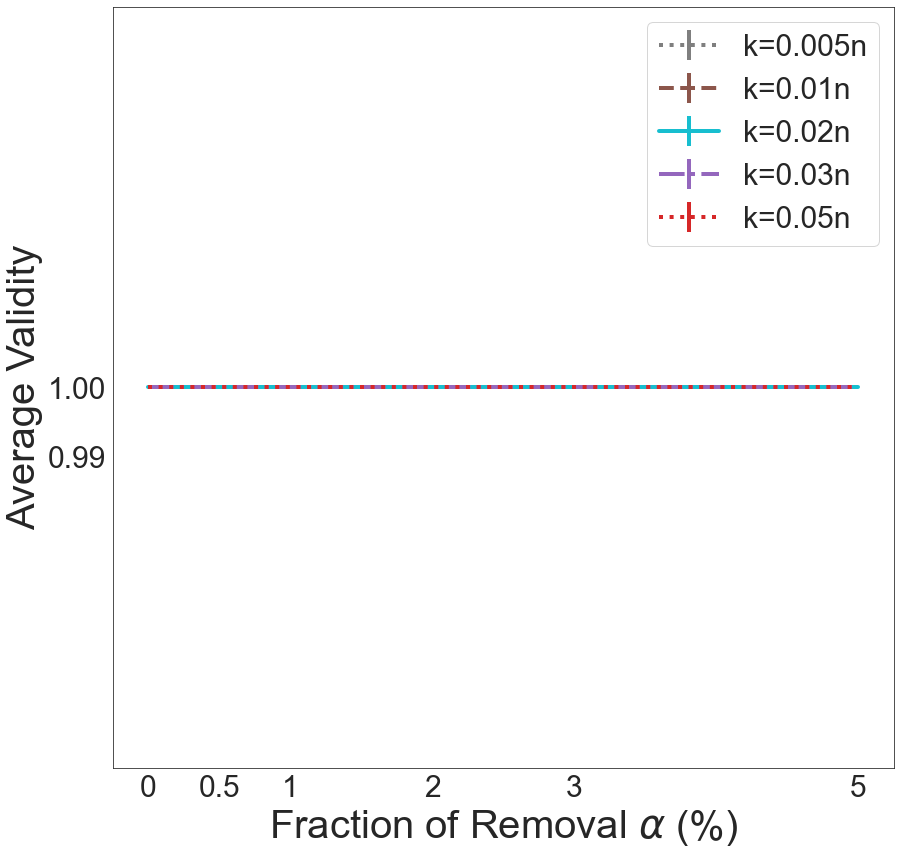}
            \caption{COMPAS.}
            \label{fig:compas-linear}
        \end{subfigure}
        \begin{subfigure}[b]{0.31\textwidth}
            \centering
            \includegraphics[width=\textwidth]{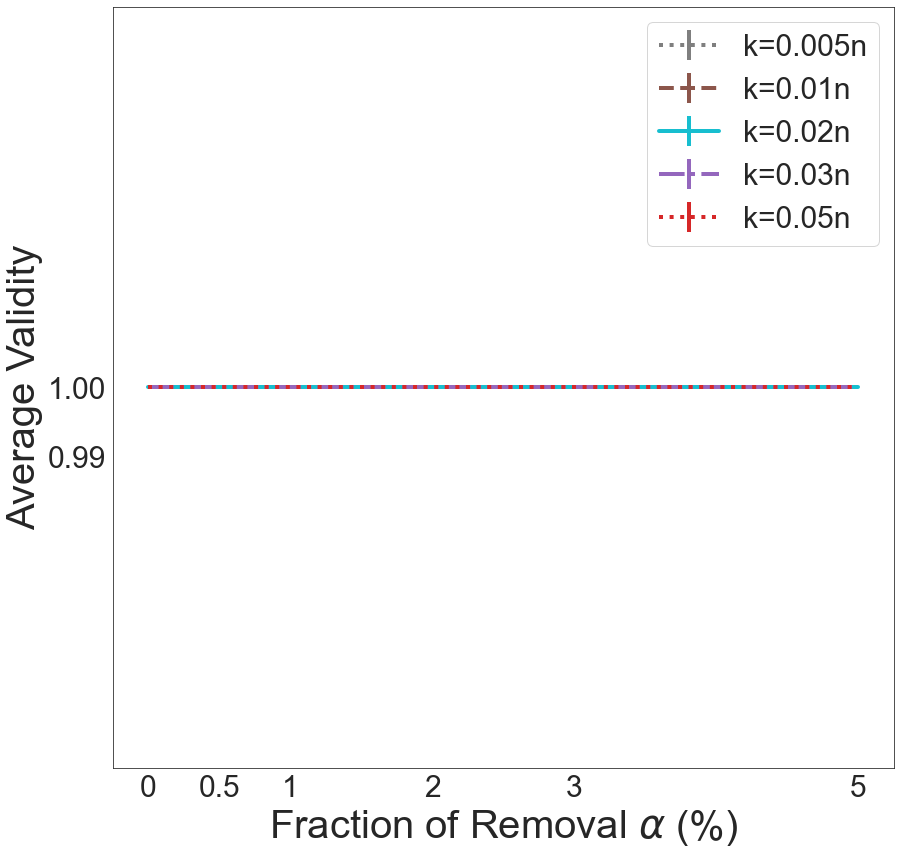}
            \caption{Adult.}
            \label{fig:adult-linear}
        \end{subfigure}
        \caption{Sensitivity analysis with respect to the hyperparameter $k$.}  %
        \label{fig:validity-linear-sensitivity}
\end{figure*}


\end{document}